\documentclass[smallextended]{svjour3}       
\smartqed  
\usepackage{graphicx}
\usepackage{algorithm}
\usepackage{algorithmicx}
\usepackage{algpseudocode}
\usepackage{bm}
\usepackage{subfigure}
\usepackage{multirow}
\usepackage{mathbbol}
\usepackage{amsmath}
\usepackage{amssymb}
\usepackage{graphicx}
\usepackage{calc}

\usepackage[sort]{natbib}

\newcommand{\argmax}{\mathop{\rm arg~max}\limits}
\newcommand{\argmin}{\mathop{\rm arg~min}\limits}
\def\Ndelta{N_{\scalebox{0.5}{$\Delta$}}}
\def\Td#1{T_{\scalebox{0.5}{$#1$}}}
\def\Tdd#1{T'_{\scalebox{0.5}{$#1$}}}
\def\E{\mathbb{E}}
\def\P{\mathbb{P}}
\def\Ev{\mathcal{E}}
\def\UB#1#2{\overline{\mu}_{#1}(#2)}
\def\LB#1#2{\underline{\mu}_{#1}(#2)}
\def\UBd#1#2{\overline{\mu}'_{#1}(#2)}
\def\LBd#1#2{\underline{\mu}'_{#1}(#2)}
\def\e{\mathrm{e}}
\def\d{\mathrm{d}}
\def\ndelta{n_{\scalebox{0.5}{$\Delta,\delta$}}}
\def\UD{\underline{\Delta}}
\def\baec#1#2#3{\mathrm{BAEC}[#1,#2,#3]}

\allowdisplaybreaks[4]

\begin{document}

\title{A Bad Arm Existence Checking Problem
}
\subtitle{How to Utilize Asymmetric Problem Structure?}


\author{Koji Tabata         \and
  Atsuyoshi Nakamura \and
  Junya Honda \and
  Tamiki Komatsuzaki
}


\institute{Koji Tabata \at
              Research Center of Mathematics for Social Creativity, Research Institute for Electronic Science,
              Hokkaido University, Kita 20 Nishi 10, Kita-ku, Sapporo 001-0020, Japan \\
              \email{ktabata@es.hokudai.ac.jp}           
              \and
              Atsuyoshi Nakamura \at
              Graduate School of Information Science and Technology,
              Hokkaido University, Kita 14, Nishi 9, Kita-ku, Sapporo 060-0814, Hokkaido, Japan
              \at 
              Research Center of Mathematics for Social Creativity, Research Institute for Electronic Science,
              Hokkaido University, Kita 20 Nishi 10, Kita-ku, Sapporo 001-0020, Japan \\
              \email{atsu@ist.hokudai.ac.jp}           
           \and
           Junya Honda \at
           University of Tokyo, 5-1-5 Kashiwanoha, Kashiwa-shi, Chiba 277-8561, Japan\\
           \email{jhonda@k.u-tokyo.ac.jp}           
                         \and
              Tamiki Komatsuzaki \at
              Research Center of Mathematics for Social Creativity, Research Institute for Electronic Science,
              Hokkaido University, Kita 20 Nishi 10, Kita-ku, Sapporo 001-0020, Japan 
              \at
              Institute for Chemical Reaction Design and Discovery (WPI-ICReDD), 
              Hokkaido University, Kita 21 Nishi 10, Kita-ku, Sapporo, Hokkaido 001-0021, Japan
              \at
              École Normale Supérieure de Lyon, 46 allée d'Italie, 69007 Lyon, France\\
              \email{tamiki@es.hokudai.ac.jp}           
}

\date{Received: date / Accepted: date}

\maketitle

\begin{abstract}
  We study a \emph{bad arm existing checking problem} in which a player's task is to judge whether a \emph{positive} arm exists or not
  among given $K$ arms by drawing as small number of arms as possible. Here, an arm is positive if its expected loss suffered by
  drawing the arm is at least a given threshold. This problem is a formalization of diagnosis of disease or machine failure.
  An interesting structure of this problem is the asymmetry of \emph{positive} and \emph{negative} (non-positive) arms' roles; finding one positive arm is enough to judge existence while all the arms must be discriminated as negative to judge non-existence.
  We propose an algorithms with \emph{arm selection policy} (policy to determine the next arm to draw) and \emph{stopping condition} (condition to stop drawing arms) utilizing this asymmetric problem structure
  and prove its effectiveness theoretically and empirically.
  
\end{abstract}
\section{Introduction}

In the diagnosis of disease or machine failure,
the test object is judged as ``positive''
if some anomaly is detected in at least one of many parts.
In the case that the purpose of the diagnosis is the classification into two classes,
``positive'' and ``negative'', then the diagnosis can be terminated right after
the first anomaly part has been detected.
Thus, fast diagnosis will be realized if one of anomaly parts can be detected as fast as possible
in positive case.

The fast diagnosis of anomaly detection is particularly important in the case that the judgment is done based on
measurements using a costly or slow device. For example, a Raman spectral image has been known to be useful for cancer diagnosis \citep{HV2009}, but its acquisition time is 1--10 seconds per point (pixel)\footnote{http://www.horiba.com/en\_en/raman-imaging-and-spectroscopy-recording-spectral-images-profiles/} resulting in an order of hours or days per one image (typically 10,000--40,000 pixels), so it is critical to measure only the points necessary for cancer diagnosis in order to achieve fast measurement. A Raman spectrum of each point is believed to be converted to a \emph{cancer index}, which indicates how likely the point is inside a cancer cell, and we can judge the existence of cancer cells from the existence of area with a high cancer index.

The above cancer cell existence checking problem can be formulated as the problem of checking the existence of a grid
with a high cancer index for a given area that is divided into grids.
By regarding each grid as an arm, we formalize this problem as a loss-version of a stochastic $K$-armed bandit problem in which the existence of \emph{positive} arms is checked by drawing arms and suffering losses for the drawn arms.
In our formulation, given an acceptable error rate $0<\delta<1/2$ and two thresholds $\theta_L$ and $\theta_U$ with $0<\theta_L<\theta_U<1$ and $\Delta=\theta_U-\theta_L$,
a player is required to, with probability at least $1-\delta$,  answer ``positive'' if positive arms exist and
``negative'' if all the arms are \emph{negative}. Here,  
an arm is defined to be positive if its loss mean is at least $\theta_U$,
and defined to be negative if its loss mean is less than $\theta_L$.
We call player algorithms for this problem as \emph{$(\Delta,\delta)$-BAEC (Bad Arm Existence Checking) algorithms}.
The objective of this research is to design a $(\Delta,\delta)$-BAEC algorithm that minimizes the number of arm draws,
that is, an algorithm with the lowest sample complexity.
The problem of this objective is said to be a \emph{Bad Arm Existence Checking Problem}.

The bad arm existence checking problem is closely related to the 
\emph{thresholding bandit problem} \citep{LGC2016}, which is a kind of pure-exploration problem such as the \emph{best arm identification problem} \citep{EMM06,ABM2010}.
In the thresholding bandit problem, provided a threshold $\theta$ and a required precision $\epsilon>0$, the player's task is to classify each arm into positive (its loss mean is at least $\theta+\epsilon$) or negative (its loss mean is less than $\theta-\epsilon$) by drawing a fixed number of samples,
and his/her objective is to minimize the error probability, that is, the probability that positive (resp. negative) arms  are wrongly classified into negative (resp. positive).
Apart from whether fixed confidence (constraint on error probability to achieve) or fixed budget (constraint on the allowable number of draws), positive and negative arms are treated symmetrically in the thresholding bandit problem while
they are dealt with asymmetrically in our problem setting; judgment of one positive arm existence is enough for positive conclusion though
all the arms must be judged as negative for negative conclusion.
This asymmetry has also been considered in the \emph{good arm identification problem} \citep{KHSMNS2017},
and our problem can be seen as its specialized version.
In their setting, the player's task is to output all the arms of above-threshold means with probability at least $1-\delta$, and his/her objective is to minimize the number of drawn samples until
$\lambda$ arms are outputted as arms with above-threshold means for a given $\lambda$.
In the case with $\lambda=1$, algorithms for their problem can be used to solve our existence checking problem.
Their proposed algorithm, however, does not utilize the asymmetric problem structure.
In this paper, we address the issue of how to utilize the structure.

We consider algorithms that are mainly composed of an \emph{arm-selection policy} and a \emph{stopping condition}. The arm-selection policy decides which arm is drawn at each time based on loss samples obtained so far. The stopping condition is used to judge whether the number of loss samples of each arm is enough to discriminate between positive and negative arms. If the currently drawn arm is judged as a positive arm, then the algorithms stop immediately by returning ``positive''. In the case that the arm is judged as a negative arm, the arm is removed from the set of positive-arm candidates, which is composed of all the arms initially,  and will not be drawn any more. If there remains no positive-arm candidate,
then the algorithms stop by returning ``negative''.

To utilize our asymmetric problem structure, we propose a stopping condition that uses $\Delta$-dependent \emph{asymmetric} confidence bounds of estimated loss means.
Here, asymmetric bounds mean that the width of the upper confidence interval is narrower than the width of the lower confidence interval,
and the algorithm using our stopping condition stops drawing each arm $i$ if its lower confidence bound of the estimated loss is at least $\theta_L$ or
its upper confidence bound is less than $\theta_U$.
As an arm selection policy, we propose policy $\mathrm{APT}_\mathrm{P}$ that is derived by modifying policy APT \citep{LGC2016} so as to favor arms with sample means larger than a single threshold $\theta$ (rather than arms with sample means closer to $\theta$ as the original APT does). Here, as the single threshold $\theta$ used by policy $\mathrm{APT}_\mathrm{P}$,
we use not the center between $\theta_L$ and $\theta_U$ but the value closer to $\theta_U$ by utilizing the asymmetric structure of our problem.

By using $\Delta$-dependent asymmetric confidence bounds as the stopping condition,
the worst-case bound on the number of samples for each arm is shown to be improved by $\Omega\left(\frac{1}{\Delta^2}\ln\frac{\sqrt{K}}{\Delta^2}\right)$ compared to the case using the conventional stopping condition of the successive elimination algorithm \citep{EMM06}.
Regarding the asymptotic behavior as $\delta\rightarrow 0$, the upper bound on the expected number of samples for our algorithm with arm selection policy $\mathrm{APT}_\mathrm{P}$ is proved to be almost optimal
when all the positive arms have the same loss mean,
which is the case that HDoC \citep{KHSMNS2017} does not perform well.
Note that HDoC is an algorithm for good arm identification that uses $\mathrm{UCB}$ \citep{ACF2002} as the arm selection policy. 
Our upper bound for $\mathrm{APT}_\mathrm{P}$ does not depend on the existence of near-optimal arms unlike that for $\mathrm{UCB}$.

The effectiveness of our stopping condition using the $\Delta$-dependent asymmetric confidence bounds is demonstrated in simulation experiments. The algorithm using our stopping condition stopped drawing an arm about two times faster than the algorithm using the conventional stopping condition when its loss mean is around the center of the thresholds.
Our algorithm with arm selection policy $\mathrm{APT}_\mathrm{P}$ always stopped faster than the algorithm using arm selection policy UCB \citep{ACF2002} like HDoC \citep{KHSMNS2017}, and our algorithm's stopping time was faster or comparable to the stopping time of the algorithm using arm selection policy LUCB \citep{KTAS2012} in our simulations
using Bernoulli loss distribution with synthetically generated means and means generated from a real-world dataset.

\section{Preliminary}
For given thresholds $0<\theta_L<\theta_U<1$, consider a following bandit problem.
Let $K(\geq 2)$ be the number of arms, and at each time $t=1,2,\dots$, a player draws arm $i_t\in \{1,\dots,K\}$.
For $i\in \{1,\dots,K\}$, $X_i(n)\in [0,1]$ denotes the loss for the $n$th draw of arm $i$, where
$X_i(1),X_i(2),\dots$ are a sequence of i.i.d. random variables generated according to a probability distribution $\nu_i$
with mean $\mu_i\in [0,1]$. We assume independence between $\{X_i(t)\}_{t=1}^{\infty}$ and $\{X_j(t)\}_{t=1}^{\infty}$ for any $i,j\in \{1,\dots,K\}$ with $i\neq j$.
For a distribution set ${\bm \nu}=\{\nu_i\}$ of $K$ arms, $\mathbb{E}_{\bm \nu}$ and $\mathbb{P}_{\bm \nu}$
denote the expectation and the probability under $\bm \nu$, respectively, and we omit
the subscript $\bm \nu$ if it is trivial from the context.
Without loss of generality, we can assume that $\mu_1\geq \cdots\geq \mu_K$ and the player does not know this ordering.
Let $n_i(t)$ denote the number of draws of arm $i$ right before the beginning of the round at time $t$. 
After the player observed the loss $X_{i_t}(n_{i_t}(t)+1)$, he/she can choose stopping or continuing to play at time $t+1$.
Let $T$ denote the stopping time.

The player's objective is to check the existence of some \emph{positive} arm(s) with as small a stopping time $T$ as possible.
Here, arm $i$ is said to be \emph{positive} if $\mu_i\geq\theta_U$, \emph{negative} if $\mu_i<\theta_L$, and \emph{neutral}
otherwise.
We consider a \emph{bad arm existence checking problem}, which is a problem of developing
algorithms that satisfy the following definition with as small number of arm draws as possible.

\begin{definition}\label{def:delta-delta-alg}
  Given\footnote{Thresholds $\theta_L$ and $\theta_U$ correspond to $\theta-\epsilon$ and $\theta+\epsilon$, respectively, in thresholding bandit problem \citep{LGC2016} with one threshold $\theta$ and precision $\epsilon$, but we use the two thresholds due to convenience for our asymmetric problem structure.} $0<\theta_L<\theta_U<1$ with $\Delta=\theta_U-\theta_L$ and $\delta\in (0,1/2)$,
    consider a game that repeats choosing one of $K$ arms and observing its loss at each time $t$.
    A player algorithm for this game is said to be a {\it $(\Delta,\delta)$-BAEC (Bad Arm Existence Checking) algorithm}
    if it stops in a finite time outputting ``positive'' with probability at least $1-\delta$ if at least one arm is positive, and ``negative'' with probability at least $1-\delta$ if all the arms are negative.
  \end{definition}

Note that the definition of BAEC algorithms requires nothing when arm $1$ is neutral.
Table~\ref{notationlist} is the table of notations used throughout this paper.

\begin{table}[tbh]
\caption{Notation List}\label{notationlist}
\hrule
  \begin{tabular}{@{}r@{}c@{}lr@{}c@{}l}
    $K$&: & \multicolumn{4}{@{}l}{Number of arms.}\\
    $\theta_U$, $\theta_L$ &: & \multicolumn{4}{@{}l}{Upper and lower thresholds. ($0< \theta_L< \theta_U<1$)}\\
    $\Delta$ &: & \multicolumn{4}{@{}l}{Gray zone width $(\Delta=\theta_U-\theta_L)$.}\\
    $\alpha$ & $=$ & \multicolumn{4}{@{}l}{$\sqrt{1+\frac{\ln K}{\ln\frac{\Ndelta}{\delta}}}$}\\
    $\theta$&$=$& \multicolumn{4}{@{}l}{$\theta_U-\frac{1}{1+\alpha}\Delta=\theta_L+\frac{\alpha}{1+\alpha}\Delta$}\\
      $\delta$&: & \multicolumn{4}{@{}l}{Acceptable error rate. ($\delta\in (0,1/2)$)}\\
      $\nu_i$&: & \multicolumn{4}{@{}l}{Loss distribution of arm $i$.}\\
      $\bm \nu$&: & \multicolumn{4}{@{}l}{Set $\{\nu_i\}$ of loss distributions of $K$ arms.}\\
    $\mu_i$&: & \multicolumn{4}{@{}l}{Loss mean (expected loss) of arm $i$. ($\mu_i\in [0,1]$)}\\
      && \multicolumn{4}{@{}l}{Arm $i$ is $\begin{cases}
                         \text{positive} & \text{if }\mu_i\geq \theta_U,\\ 
                         \text{neutral} & \text{if }\theta_L \leq  \mu_i< \theta_U,\\
                         \text{negative} & \text{if }\mu_i< \theta_L.
      \end{cases}$}\\
      $\E_{\bm \nu}$&: & \multicolumn{4}{@{}l}{Expectation of some random variable w.r.t. $\bm \nu$.}\\
      $\P_{\bm \nu}$&: & \multicolumn{4}{@{}l}{Probability of some event w.r.t. $\bm \nu$.}\\
       &&  \multicolumn{4}{@{}l}{($\bm \nu$ is omitted when it is trivial from the context.)}\\
      $i_t$&: & \multicolumn{4}{@{}l}{Drawn arm at time $t$.}\\
      $X_i(n)$&: & \multicolumn{4}{@{}l}{Loss suffered by the $n$th draw of arm $i$.}\\
      $n_i(t)$&: & \multicolumn{4}{@{}l}{Number of draws of arm $i$ at the beginning of the round at time $t$.}\\
      $T$&: & \multicolumn{4}{@{}l}{Stopping time.}\\

    $\hat{\mu}_i(n)$ & $=$ & $\frac{1}{n}\sum_{s=1}^nX_i(s)$ &
    $\Ndelta$ & $=$ & $\left\lceil\frac{2\e}{(\e-1)\Delta^2}\ln\frac{2\sqrt{K}}{\Delta^2\delta}\right\rceil$\\
  $\LB{i}{n}$ & $=$ &  $\hat{\mu}_i(n) - \sqrt{\frac{1}{2 n} \ln{\frac{K\Ndelta}{\delta}} }$ &
  $\UB{i}{n}$ & $=$ &  $\hat{\mu}_i(n) + \sqrt{\frac{1}{2 n} \ln{\frac{\Ndelta}{\delta}} }$\\
  $\Td{\Delta}$ & $=$ & $\left\lceil\frac{2}{\Delta^2}\ln\frac{\sqrt{K}\Ndelta}{\delta}\right\rceil$ &
  $\Delta_i$ &$=$& $\begin{cases}
    \mu_i-\theta_L & (\mu_i\geq \theta)\\
    \theta_U-\mu_i & (\mu_i<\theta)
    \end{cases}$\\
  $\Td{\Delta_i}$&$=$&\multicolumn{4}{@{}l}{$\left\lceil\frac{2}{\Delta_i^2}\ln\frac{\sqrt{K}\Ndelta}{\delta}\right\rceil$}\\
    $\tau_i$ & : & \multicolumn{4}{p{8.5cm}}{Number $n$ of draws of arm $i$ until
  algorithm~$\baec{\ast}{\underline{\mu}}{\overline{\mu}}$'s stopping condition ($\LB{i}{n}\geq \theta_L$ or $\UB{i}{n}<\theta_U$) is satisfied.}\\
  $\hat{i}_1$ & : & \multicolumn{4}{@{}l}{First arm that is drawn $\tau_i$ times by algorithm~$\baec{\mathrm{APT}_\mathrm{P}}{\underline{\mu}}{\overline{\mu}}$}\\
  $\Ev^+$ & $=$ &$\bigcup_{i:\mu_i\geq \theta_U}\bigcap_{n=1}^{\Td{\Delta}}\Bigl\{\UB{i}{n}\geq \mu_i\Bigr\}$ &
  $\Ev^-$ & $=$ &$\bigcap_{i=1}^K\bigcap_{n=1}^{\Td{\Delta}}\Bigl\{\LB{i}{n}< \mu_i\Bigr\}$\\
  $\UD_i$&$=$&\multicolumn{4}{@{}l}{$|\mu_i-\theta|$}\\
        $m$ & : &\multicolumn{4}{@{}l}{Number of arms $i$ with $\mu_i\geq \theta$.}\\
  $\Ev_i^\mathrm{POS}$ & : &\multicolumn{4}{@{}l}{Event that arm $i$ is judged as positive.}\\
  $\Delta_{1i}$ & $=$ &\multicolumn{4}{@{}l}{$\mu_1-\mu_i$}
  \end{tabular}  
  \hrule
  \end{table}

\section{Sample Complexity Lower Bound}

In this section, we derive a lower bound on the expected number of samples needed for a $(\Delta,\delta)$-BAEC algorithm.
The derived lower bound is used to evaluate algorithm's sample complexity upper bound in Sec.~\ref{sec:APT-BCbound} and Sec.~\ref{sec:CompWithUCB}.

We let $\mathrm{KL}(\nu,\nu')$ denote Kullback-Leibler divergence
from distribution $\nu'$ to $\nu$
and define $d(x,y)$ as
\[
d(x,y)=x\ln\frac{x}{y}+(1-x)\ln\frac{1-x}{1-y}.
\]
Note that $\mathrm{KL}(\nu,\nu')=d(\mu_i,\mu'_i)$ holds if $\nu$ and $\nu'$ are Bernoulli distributions with means $\mu_i$ and $\mu'_i$, respectively.

\begin{theorem}\label{thlowerboud}
  Let $\{\nu_i\}$ be a set of Bernoulli distributions with means $\{\mu_i\}$.
Then, the stopping time $T$ of any $(\Delta,\delta)$-BAEC algorithm with $\theta_U$ and $\theta_L$ is bounded as
  \begin{align}
    \E(T)>\frac{1-2\delta}{d(\mu_1,\theta_L)}\ln\frac{1-\delta}{\delta} \label{lower-bound-positive}
  \end{align}
 if some arm is positive,
and 
  \begin{align}
    \E(T)> \sum_{i=1}^K\frac{1-2\delta}{d(\mu_i,\theta_U)}\ln\frac{1-\delta}{\delta}\label{lower-bound-negative}
  \end{align}
if all the arms are negative.
\end{theorem}

\begin{proof}
See Appendix~\ref{appendix:theorem1}.
\hspace*{\fill}$\Box$\end{proof}

\begin{remark}
Identification is not needed for checking existence, however, in terms of asymptotic behavior as $\delta\rightarrow +0$,
the shown expected sample complexity lower bounds of both the tasks are the same;
$\lim_{\delta\rightarrow +0}\E(T)/\ln (1/\delta)\geq 1/d(\mu_1,\theta_L)$ for both the tasks in the case with some positive arms.
The bounds are tight considering the shown upper bounds, 
so the bad arm existence checking is not more difficult than the good arm identification \citep{KHSMNS2017} with respect to asymptotic behavior as $\delta\rightarrow +0$. 
\end{remark}

\section{Algorithm}

\begin{algorithm}[tb]
\caption{$\baec{\mathrm{ASP}}{\mathrm{LB}}{\mathrm{UB}}$}
\label{alg:BAEC}
\footnotesize
  \textbf{Parameter Function:}\\
$\phantom{\text{\textbf{Input:}}}$  \begin{minipage}[t]{11cm}
    $\mathrm{ASP}(t,i)$: index value of arm $i$ at time $t$ for arm selection\\
    $\mathrm{LB}(t)$, $\mathrm{UB}(t)$: lower and upper confidence bounds of arm $i_t$'s estimated loss mean
  \end{minipage}\\
  \textbf{Input:}
  \begin{minipage}[t]{7cm}
    $K$: the number of arms\\
    $0<\theta_L<\theta_L<1$: thresholds with $\Delta=\theta_U-\theta_L$\\
    $\delta\in (0,1/2)$: acceptable error rate
  \end{minipage}
\begin{algorithmic}[1]
\State $A_1 \gets \{1, 2, \ldots, K\}$, $\Ndelta \gets \left\lceil\frac{2e}{(e-1) \Delta^2} \ln{\frac{2\sqrt{K}}{\delta \Delta^2}}\right\rceil $
\For{$i \in A_1$}
	\State $n_i(1) \gets 0$, $\hat{\mu}_i(0) \gets \theta$
\EndFor
\State $t \gets 1$
\While{$A_t\ne \emptyset$}
	\State $i_t \gets \argmax_{i\in A_t}\mathrm{ASP}(t,i)$ \label{alg:arm-selection}
	\State $n_i(t+1) \gets \begin{cases}
          n_i(t) + 1 & (i=i_t)\\
          n_i(t) & (i\neq i_t)
        \end{cases}$
       	\State Draw $i_t$ and suffer a loss $X_{i_t}(n_{i_t}(t+1))$. \label{alg:loss}
	\State $\hat{\mu}_{i_t}(n_{i_t}(t+1)) \gets \frac{\hat{\mu}_{i_t}(n_{i_t}(t)) \times n_{i_t}(t) + X_{i_t}(n_{i_t}(t+1))}{n_{i_t}(t+1)}$
	\If{$\mathrm{LB}(t)\geq \theta_L$} \label{alg:if}
        \State \Return \mbox{``positive''}  \Comment{Arm $i_t$ is judged as pos.} 
	\ElsIf{$\mathrm{UB}(t)<\theta_U$} \label{alg:elseif}
		\State $A_{t+1} \gets A_t \setminus \{i_t\}$ \Comment{Arm $i_t$ is judged as neg.} 
	\EndIf
        \State $t \gets t + 1$        
\EndWhile
\State \Return \mbox{``negative''}
\end{algorithmic}
\end{algorithm}

As $(\Delta,\delta)$-BAEC algorithms, we consider algorithm $\baec{\mathrm{ASP}}{\mathrm{LB}}{\mathrm{UB}}$ shown in Algorithm~\ref{alg:BAEC} that, at each time $t$,  chooses an arm $i_t$ from the set $A_t$ of positive-candidate arms by an \emph{arm-selection policy} $\mathrm{ASP}$
\[
i_t \gets \argmax_{i\in A_t} \mathrm{ASP}(t,i) 
\]
using some index value $\mathrm{ASP}(t,i)$ of arm $i$ at time $t$ (Line \ref{alg:arm-selection}), suffers a loss $X_{i_t}(n_{i_t}(t+1))$ (Line~\ref{alg:loss}) and then checks whether a \emph{stopping condition}
\[
\mathrm{LB}(t)\geq\theta_L \ \ \text{or}\ \ \mathrm{UB}(t)<\theta_U
\]
is satisfied (Lines~\ref{alg:if} and \ref{alg:elseif}).
Here, $\mathrm{LB}(t)$ and $\mathrm{UB}(t)$ are lower and upper confidence bounds of an estimated loss mean of the current drawn arm $i_t$,
and condition $\mathrm{LB}(t)\geq \theta_L$ is the condition for stopping drawing any arm and outputting ``positive'',
and condition $\mathrm{UB}(t)<\theta_U$ is the condition for stopping drawing arm $i_t$ concluding its negativity
and removing $i_t$ from the set $A_{t+1}$ of positive-candidate arms of time $t+1$.
In addition to the case with outputting ``positive'', algorithm $\baec{\mathrm{ASP}}{\mathrm{LB}}{\mathrm{UB}}$ also stops outputting ``negative''
when $A_t$ becomes empty.

Define sample loss mean $\hat{\mu}_i(n)$ of arm $i$ with $n$ draws as
\[
\hat{\mu}_i(n) = \frac{1}{n}\sum_{s=1}^nX_i(s),
\]
and we use $\hat{\mu}_{i_t}(n_{i_t}(t+1))$ as an estimated loss mean of the current drawn arm $i_t$ at time $t$.
Thus, $\mathrm{LB}(t)$ and $\mathrm{UB}(t)$ are determined by defining lower and upper bounds of a confidence interval of $\hat{\mu}_i(n)$
for $i=i_t$ and $n=n_{i_t}(t+1)$.

As lower and upper confidence bounds of $\hat{\mu}_i(n)$,
\begin{align}
\LBd{i}{n}=\hat{\mu}_i(n)-\sqrt{\frac{1}{2n}\ln\frac{2Kn^2}{\delta}}  \text{ and }
\UBd{i}{n}=\hat{\mu}_i(n)+\sqrt{\frac{1}{2n}\ln\frac{2Kn^2}{\delta}},
\label{lbd} 
\end{align}
respectively, are generally used\footnote{Precisely speaking, $\hat{\mu}_i(n)\pm\sqrt{\frac{1}{2n}\ln\frac{4Kn^2}{\delta}}$ is used in successive elimination algorithms for best arm identification problem.
  A narrower confidence interval is enough to judge whether expected loss is larger than a \emph{fixed} threshold.} in successive elimination algorithms \citep{EMM06}.
Define $\underline{\mu}'(t)$ and $\overline{\mu}'(t)$ as $\underline{\mu}'(t)=\LBd{i_t}{n_{i_t}(t+1)}$ and $\overline{\mu}'(t)=\UBd{i_t}{n_{i_t}(t+1)}$ for use as $\mathrm{LB}(t)$ and $\mathrm{UB}(t)$. 

In this paper, we propose asymmetric bounds $\LB{i}{n}$ and $\UB{i}{n}$ defined using a gray zone width $\Delta=\theta_U-\theta_L$
as follows:
\begin{align}
  \LB{i}{n}\ =  \hat{\mu}_i(n) - \sqrt{\frac{1}{2 n} \ln{\frac{K\Ndelta}{\delta}} }  \text{ and }
  \UB{i}{n}\ =  \hat{\mu}_i(n) + \sqrt{\frac{1}{2 n} \ln{\frac{\Ndelta}{\delta}} },
\label{lb}
\end{align}
where 
\[
\Ndelta=\left\lceil\frac{2\e}{(\e-1)\Delta^2}\ln\frac{2\sqrt{K}}{\Delta^2\delta}\right\rceil.
\]
We also let $\underline{\mu}(t)$ and $\overline{\mu}(t)$ denote $\mathrm{LB}(t)$ and $\mathrm{UB}(t)$ using these bounds,
that is, $\underline{\mu}(t)=\LB{i_t}{n_{i_t}(t+1)}$ and $\overline{\mu}(t)=\UB{i_t}{n_{i_t}(t+1)}$.

The idea of our bounds are derived as follows.
By using lower bound $\hat{\mu}_i(n)-\sqrt{\frac{1}{2n}\ln \frac{1}{\delta a_n}}$,
$\P\left[\bigcup_{n=1}^\infty\{\hat{\mu}_i(n)-\sqrt{\frac{1}{2n}\ln \frac{1}{\delta a_n}}>\mu_i\}\right]$ is upper bounded by $\delta\sum_{n=1}^\infty a_n$.
This can be proved using Hoeffding's Inequality and the union bound.
The conventional bound $\LBd{i}{n}$ uses decreasing sequence $a_n=\frac{1}{2Kn^2}$
while our bound $\LB{i}{n}$ uses a constant sequence $a_n=\frac{1}{K\Ndelta}$.
Even though $\sum_{n=1}^\infty a_n=\infty$ for such constant sequence $a_n=\frac{1}{K\Ndelta}$,
$\P\left[\bigcup_{t}\{\LB{i}{n_i(t)}>\mu_i\}\right]$ can be upper bounded by $\delta/K$
because stopping condition is satisfied for any arm $i$ and any $n_i(t)>\Ndelta$, which is derived from Lemma~\ref{lemUtBjbW9v} and Proposition~\ref{prop:2}. Note that $\Ndelta$ depends on gray zone width $\Delta$, and the larger the $\Delta$ is, the smaller the $\Ndelta$ is. 
Our upper bound $\UB{i}{n}$ is closer to $\hat{\mu}_i$ than $\LB{i}{n}$, that is, the positions of $\LB{i}{n}$ and $\UB{i}{n}$ are
not symmetric with respect to the position of $\hat{\mu}_i$. This is a reflection of our asymmetric problem setting.
In the case with $\mu_1<\theta_L$, any arm must not be judged as positive ($\LB{i}{n}\geq \theta_L$ for some $n$) for correct conclusion,
so the probability of wrongly judged as positive for each arm must be at most $\delta/K$ for the union bound.
On the other hand, in the case with $\mu_1\geq \theta_U$, correct judgment for arm $1$ is enough for correct conclusion,
so the probability of wrongly judged as negative ($\UB{i}{n}< \theta_U$ for some $n$) for each positive arm $i$ can be at most $\delta$.

Note that $\UB{i}{n}-\LB{i}{n}<\UBd{i}{n}-\LBd{i}{n}$ holds for $n\geq \sqrt{\Ndelta/2}$.
Both $\UB{i}{n}-\LB{i}{n}$ and $\UBd{i}{n}-\LBd{i}{n}$ decrease as $n$ increases, and
$\mathrm{LB}(t)\geq \theta_L$ or $\mathrm{UB}(t)<\theta_U$
is satisfied for $\baec{\ast}{\underline{\mu}}{\overline{\mu}}$ and $\baec{\ast}{\underline{\mu}'}{\overline{\mu}'}$ when they become at most $\Delta$ for $n=n_i(t+1)$, where $\mathrm{ASP}=\ast$ means that any index function $\mathrm{ASP}(t,i)$ can be assumed.

\begin{remark}
  Condition $\underline{\mu}(t)\geq \theta_L$ essentially identifies non-negative arm $i_t$.
  Is there real-valued function $\mathrm{LB}$ that can check existence of a non-negative arm without identifying it?
  The answer is yes. Consider a virtual arm at each time $t$ whose mean loss $\mu^t$ is a weighted average over the mean losses $\mu_i$ of all the arms $i$ ($i=1,\dots,K$) defined as $\mu^t=\frac{1}{t}\sum_{i=1}^Kn_i(t+1)\mu_i$.
  If $\mu^t\geq \theta_L$, then at least one arm $i$ must be non-negative.
  Thus, we can check the existence of a non-negative arm by judging whether $\mu^t\geq \theta_L$ or not.
  Since $\underline{\mu}^t(t)$ defined as
  \[
  \underline{\mu}^t(t)=\frac{1}{t}\sum_{i=1}^Kn_i(t+1)\hat{\mu}_i(t+1)-\sqrt{\frac{1}{2t}\ln\frac{2t^2}{\delta}}
  \]
  can be considered to be a lower bound of the estimated value of $\mu^t$,
  $\underline{\mu}^t$ can be used as $\mathrm{LB}$ for checking the existence of a non-negative arm without identifying it.
\end{remark}

The ratio of the width of our upper confidence interval $\left[\hat{\mu}_i(n),\UB{i}{n}\right]$ to  the width of our lower confidence interval $\left[\LB{i}{n},\hat{\mu}_i(n)\right]$ is $\sqrt{\ln\frac{\Ndelta}{\delta}}:\sqrt{\ln\frac{K\Ndelta}{\delta}}=1:\sqrt{1+\frac{\ln K}{\ln\frac{\Ndelta}{\delta}}}$.
Thus, we define $\theta$ as
\[
\theta=\theta_U-\frac{1}{1+\alpha}\Delta \ \ \text{ where } \alpha=\sqrt{1+\frac{\ln K}{\ln\frac{\Ndelta}{\delta}}}.
\]
This $\theta$ can be considered to be the balanced center between the thresholds $\theta_L$ and $\theta_U$ for our asymmetric confidence bounds.

As arm selection policy $\mathrm{ASP}$, we consider policy $\mathrm{APT}_{\mathrm{P}}$ that uses index function
\begin{align}
\mathrm{APT}_\mathrm{P}(t,i)=\sqrt{n_i(t)} \left(\hat{\mu}_i(n_i(t)) - \theta\right). \label{apt-bc-index}
\end{align}
This arm-selection policy is a modification of the policy of $\mathrm{APT}$ {\footnotesize (Anytime Parameter-free Thresholding algorithm)} \citep{LGC2016}, in which an arm
\[
\argmin_i \sqrt{n_i(t)} \left(\left| \hat{\mu}_i(n_i(t)) - \theta\right|+\epsilon\right)
\]
is chosen for given threshold $\theta$ and accuracy $\epsilon$.
In the original APT, arm $i$ with the sample mean $\hat{\mu}_i(n_i(t))$ closest to $\theta$ is preferred to be chosen no matter whether $\hat{\mu}_i(n_i(t))$ is larger or smaller than $\theta$.
In $\mathrm{APT}_{\mathrm{P}}$,  there is at most one arm $i$ whose sample mean $\hat{\mu}_i(n_i(t))$ is larger than $\theta$ at any time $t$ due to the initialization that $\hat{\mu}_j(0)=\theta$ for all arms $j$, and such unique arm $i$ is always chosen as long as $\hat{\mu}_i(n_i(t))>\theta$.

\section{Sample Complexity Upper Bounds}

In this section, we first analyze sample complexity of algorithm $\baec{\ast}{\underline{\mu}}{\overline{\mu}}$, then analyze sample complexity of algorithm $\baec{\mathrm{APT}_\mathrm{P}}{\underline{\mu}}{\overline{\mu}}$.

We let $\tau_i$ denote the smallest number $n$ of draws of arm $i$ for which either $\LB{i}{n}\geq\theta_L$ or $\UB{i}{n}<\theta_U$ holds.
We define $\Delta_i$ as
\[
\Delta_i=\begin{cases}
\mu_i-\theta_L & (\mu_i\geq \theta)\\
\theta_U-\mu_i & (\mu_i<\theta)
\end{cases}
\]
and let $T_x$ denote $\left\lceil\frac{2}{x^2}\ln\frac{\sqrt{K}\Ndelta}{\delta}\right\rceil$ for $x=\Delta, \Delta_i$.
We define event $\Ev^+$ and $\Ev^-$as 
  \begin{align*}
    \Ev^+=\bigcup_{i:\mu_i\geq \theta_U}\bigcap_{n=1}^{\Td{\Delta}}\Bigl\{\UB{i}{n}\geq \mu_i\Bigr\}, \text{\ }
    \Ev^-=\bigcap_{i=1}^K\bigcap_{n=1}^{\Td{\Delta}}\Bigl\{\LB{i}{n}< \mu_i\Bigr\}.
 \end{align*}
Note that algorithm $\baec{\ast}{\underline{\mu}}{\overline{\mu}}$ returns ``positive'' under the event $\Ev^+$ and
returns ``negative'' under the event $\Ev^-$.
For any event $\Ev$, we let $\mathbb{1}\{\Ev\}$ denote an indicator function of $\Ev$, that is, $\mathbb{1}\{\Ev\}=1$ if $\Ev$ occurs and $\mathbb{1}\{\Ev\}=0$ otherwise.
  
\subsection{Sample Complexity of Algorithm $\baec{\ast}{\underline{\mu}}{\overline{\mu}}$}

In this subsection, we prove that
algorithm $\baec{\ast}{\underline{\mu}}{\overline{\mu}}$ is a $(\Delta, \delta)$-BAEC algorithm.
We also show three upper bounds of the number of samples needed for algorithm $\baec{\ast}{\underline{\mu}}{\overline{\mu}}$:
a worst-case bound, a high-probability bound and an average-case bound.

A worst-case upper bound $K\Td{\Delta}$ on the number of samples is directly derived from the following theorem, which says,
the number of draws for each arm $i$ can be upper bounded by constant number $\Td{\Delta}=\left\lceil\frac{2}{\Delta^2}\ln\frac{\sqrt{K}\Ndelta}{\delta}\right\rceil$ depending on $\Delta$ and $\delta$ due to gray zone width $\Delta>0$.

\begin{theorem}\label{lemUtBjbW9v}
Inequality $\tau_i\le \Td{\Delta}$ holds for $i=1,\dots,K$.
\end{theorem}
\begin{proof}
See Appendix~\ref{appendix:lemma1}.
\hspace*{\fill}$\Box$\end{proof}

How good is the worst case bound $\Td{\Delta}$ on the number of samples for each arm
comparing to the case with $\mathrm{LB}=\underline{\mu}'$ and $\mathrm{UB}=\overline{\mu}'$?
We know from the following theorem that, in $\baec{\ast}{\underline{\mu}'}{\overline{\mu}'}$,
the number of arm draws $\tau'_i$ for some arm $i$ can be larger than $\Tdd{\Delta}=\lfloor\frac{2}{\Delta^2}\ln\frac{448K}{\Delta^4\delta}\rfloor$, which means $\tau'_i-\tau_i=\Omega\left(\frac{1}{\Delta^2}\ln\frac{\sqrt{K}}{\Delta^2}\right)$ if $\frac{1}{\delta}=o\left(\e^{\sqrt{K}/\Delta^2}\right)$.

\begin{theorem}\label{th:remark1}
  Consider algorithm $\baec{\ast}{\underline{\mu}'}{\overline{\mu}'}$
and define $\tau'_i=\min\{n\mid \LBd{i}{n}\geq \theta_L \text{ or } \UBd{i}{n}<\theta_U\}$ for $i=1,\dots,K$.
Then, event $\tau'_i > \Tdd{\Delta}$ can happen for $i=1,\dots,K$,
where $\Tdd{\Delta}$ is defined as $\Tdd{\Delta}=\lfloor\frac{2}{\Delta^2}\ln\frac{448K}{\Delta^4\delta}\rfloor$.
Furthermore, the difference between the worst case stopping times $\tau'_i-\tau_i$ is lower-bounded as
\[
\tau'_i-\tau_i> \Tdd{\Delta}-\Td{\Delta}> \frac{2}{\Delta^2}\left(\ln\frac{52\sqrt{K}}{\Delta^2}-\ln\ln\frac{3\sqrt{K}}{\Delta^2\delta}\right).
\]
\end{theorem}
\begin{proof}
See Appendix~\ref{appendix:remark1}.
\hspace*{\fill}$\Box$\end{proof}

\begin{remark}\label{remark:1}
In the experimental setting of Sec.~\ref{exp:stopcond}, in which parameters $K=100$, $\Delta=0.2$ and $\delta=0.01, 0.001$ are used, the lower bounds of the difference between the worst case stopping times $\tau'_i$ and $\tau_i$ calculated using the above inequality are $352.7$ and $343.4$, respectively, which seem relatively large compared to corresponding $\Td{\Delta}=684$ and $808$.
\end{remark}

The following theorem states that algorithm $\baec{\ast}{\underline{\mu}}{\overline{\mu}}$ is a $(\Delta, \delta)$-BAEC algorithm which needs at most $K\Td{\Delta}$ samples in the worst case.

\begin{theorem}\label{thwelldef}
Algorithm $\baec{\ast}{\underline{\mu}}{\overline{\mu}}$ is a $(\Delta, \delta)$-BAEC algorithm that stops after at most $K\Td{\Delta}$ arm draws. 
\end{theorem}
\begin{proof}
See Appendix~\ref{app:thwelldef}. 
\hspace*{\fill}$\Box$\end{proof}

A high-probability upper bound of the number of samples needed for algorithm $\baec{\ast}{\underline{\mu}}{\overline{\mu}}$ is shown in the next theorem. Compared to worst case bound, $K\Td{\Delta}$ can be improved to $\sum_{i=1}^K\Td{\Delta_i}$ in the case with $\mu_1<\theta_L$, however, only one $\Td{\Delta}$ is guaranteed to be improved to the maximum $\Td{\Delta_i}$ among those of positive arms $i$ in the case with $\mu_1\geq \theta_U$.

\begin{theorem}\label{theorem3}
In algorithm $\baec{\ast}{\underline{\mu}}{\overline{\mu}}$, inequality $\tau_i \leq \Td{\Delta_i}$ holds for at least one positive arm $i$ 
with probability at least $1-\delta$ when $\mu_1\geq \theta_U$.
Inequality $\tau_i \leq \Td{\Delta_i}$ holds for all the arm $i=1,\dots,K$ 
with probability at least $1-\delta$ when $\mu_1< \theta_L$.
As a result, with probability at least $1-\delta$, the stopping time $T$ of algorithm $\baec{\ast}{\underline{\mu}}{\overline{\mu}}$ is upper bounded as $T\leq \max_{i:\mu_i\geq \theta_U}\Td{\Delta_i}+(K-1)\Td{\Delta}$ when $\mu_1\geq \theta_U$ and $T\leq\sum_{i=1}^K\Td{\Delta_i}$ when $\mu_1<\theta_L$.
\end{theorem}
\begin{proof}
See Appendix~\ref{appendix:theorem3}.
\hspace*{\fill}$\Box$\end{proof}

The last sample complexity upper bound for algorithm $\baec{\ast}{\underline{\mu}}{\overline{\mu}}$ is an upper bound on the expected number of samples. Compared to the high-probability bound, $\Td{\Delta_i}=\left\lceil\frac{2}{\Delta_i^2}\ln\frac{\sqrt{K}\Ndelta}{\delta}\right\rceil$ is improved to $\frac{1}{2\Delta_i^2}\ln\frac{K\Ndelta}{\delta}$ or $\frac{1}{2\Delta_i^2}\ln\frac{\Ndelta}{\delta}$.

\begin{theorem}
For algorithm $\baec{\ast}{\underline{\mu}}{\overline{\mu}}$, the expected value of $\tau_i$ of each arm $i$ is upper bounded as follows. 
\[
\E\left[\tau_i\right]\leq
\begin{cases}
\frac{1}{2 \Delta_i^2} \ln{\frac{K\Ndelta}{\delta}}+O\left(\left(\ln{\frac{K\Ndelta}{\delta}}\right)^{2/3}\right) & (\mu_i\geq \theta)\\
\frac{1}{2 \Delta_i^2} \ln{\frac{\Ndelta}{\delta}}+O\left(\left(\ln{\frac{\Ndelta}{\delta}}\right)^{2/3}\right) & (\mu_i< \theta)
\end{cases}
\]
As a result, the expected stopping time $\E[T]$ of algorithm $\baec{\ast}{\underline{\mu}}{\overline{\mu}}$ is upper bounded as
\begin{align}
  \E[T] \leq  \frac{1}{2} \ln{\frac{\Ndelta}{\delta}}\sum_{i=1}^K\frac{1}{\Delta_i^2}+\frac{\ln K}{2}\sum_{i:\mu_i\geq \theta}\frac{1}{\Delta_i^2}
  +O\left(K\left(\ln{\frac{K\Ndelta}{\delta}}\right)^{2/3}\right).\label{th:est}
\end{align}
\end{theorem}

The above theorem can be easily derived from the following lemma by setting event $\Ev$ to a \emph{certain event} (an event that occurs with probability $1$).

\begin{lemma}\label{lem:expectedtau}
  For any event $\Ev$, in algorithm $\baec{\ast}{\underline{\mu}}{\overline{\mu}}$, inequality 
  \begin{align}
  \E[\tau_i\mathbb{1}\{\Ev\}]\leq \frac{\P[\Ev]}{2\Delta_i^2}\ln\frac{K\Ndelta}{\delta}+O\left(\left(\ln\frac{K\Ndelta}{\delta}\right)^{\frac{2}{3}}\right). \label{eq:lem1}
  \end{align}  
  holds for any arm $i$ with $\mu_i\geq \theta$ and 
  \begin{align}
  \E[\tau_i\mathbb{1}\{\Ev\}]\leq \frac{\P[\Ev]}{2\Delta_i^2}\ln\frac{\Ndelta}{\delta}+O\left(\left(\ln\frac{\Ndelta}{\delta}\right)^{\frac{2}{3}}\right). \label{eq:lem1(2)}
  \end{align}  
  holds for any arm $i$ with $\mu_i< \theta$.  
\end{lemma}
\begin{proof}
  See Appendix~\ref{app:lem:expectedtau}. 
\hspace*{\fill}$\Box$\end{proof}

\begin{remark}
  When all the arms have Bernoulli loss distributions with means less than $\theta_L$,
  by Pinsker's Inequality $d(x,y)\geq 2(x-y)^2$,
  the right-hand side of Ineq.~(\ref{lower-bound-negative}) in Theorem~\ref{thlowerboud} can be upper bounded as
\begin{align*}
   \sum_{i=1}^K\frac{1-2\delta}{d(\mu_i,\theta_U)}\ln\frac{1-\delta}{\delta}
  \le   \sum_{i=1}^K\frac{1-2\delta}{2\Delta_i^2}\ln\frac{1-\delta}{\delta}.
\end{align*}
Since Pinsker's Inequality is tight in the worst case,
algorithm~$\baec{\ast}{\underline{\mu}}{\overline{\mu}}$
is almost asymptotically optimal as $\delta\rightarrow +0$.
\end{remark}

\subsection{Sample Complexity of $\baec{\mathrm{APT}_\mathrm{P}}{\underline{\mu}}{\overline{\mu}}$}\label{sec:APT-BCbound}

If all the arms are judged as negative in algorithm $\baec{\mathrm{ASP}}{\mathrm{LB}}{\mathrm{UB}}$, that is, drawing arm $i$ is stopped by the stopping condition of $\mathrm{UB}(t)<\theta_U$ for all $i=1,\dots,K$, arm-selection policy $\mathrm{ASP}$ does not affect the stopping time.
In the case that some positive arms exist, however, the stopping time depends on how fast the $(\Delta,\delta)$-BAEC algorithm can find one of positive arms.

In this subsection, we prove upper bounds on the expected number of samples needed for algorithm $\baec{\mathrm{APT}_\mathrm{P}}{\underline{\mu}}{\overline{\mu}}$, an instance of algorithm $\baec{\ast}{\underline{\mu}}{\overline{\mu}}$ with specific arm-selection policy $\mathrm{APT}_\mathrm{P}$.

Let arm $\hat{i}_1$ denote the first arm that is drawn $\tau_i$ times in algorithm $\baec{\mathrm{APT}_\mathrm{P}}{\underline{\mu}}{\overline{\mu}}$.
In addition to $\Delta_i$, we also use $\UD_i=|\mu_i-\theta|$ in the following analysis.
We let $m$ denote the number of arms $i$ with $\mu_i\geq \theta$.
The event that arm $i$ is judged as positive is denoted as $\Ev_i^\mathrm{POS}$.

From the following theorem and corollary, we know that, when $\delta$ is small, the dominant terms of our upper bound on the expected stopping time of 
algorithm $\baec{\mathrm{APT}_\mathrm{P}}{\underline{\mu}}{\overline{\mu}}$, are
$\frac{\P\left[\hat{i}_1=i,\Ev_i^\mathrm{POS}\right]}{2\Delta_i^2}\ln\frac{1}{\delta}$ ($i=1,...,m$),
whose sum is between $\frac{1}{2\Delta_1^2}\ln\frac{1}{\delta}$ and $\frac{1}{2\Delta_m^2}\ln\frac{1}{\delta}$.

\begin{theorem}\label{th:apt-bc}
If $m\geq 1$ (or $\mu_1\geq \theta$), then the expected stopping time $\E[T]$ of algorithm $\baec{\mathrm{APT}_\mathrm{P}}{\underline{\mu}}{\overline{\mu}}$ is upper bounded as 
  \begin{align*}
\E[T] \leq &
    \sum_{i=1}^m\left(\frac{\P\left[\hat{i}_1=i,\Ev_i^\mathrm{POS}\right]}{2\Delta_i^2}\ln\frac{K\Ndelta}{\delta}
    +\frac{2(m-1)}{\UD_i^4}+\left(\frac{1}{\UD_i^2}+4\right)\sum_{j=m+1}^K \frac{1}{\UD_j^2}\right)\\
  & +m(K-m)+O\left(m\left(\ln\frac{K\Ndelta}{\delta}\right)^{\frac{2}{3}}\right)\\
  & +K\Td{\Delta}\Biggl(\begin{aligned}[t]&
    \frac{\e^{2\UD_i^2}}{2\UD_i^2}\sum_{i=1}^m\left(\frac{\delta}{\Ndelta}\right)^{\left(\frac{\UD_i}{\max\{\theta_U,1-\theta_L\}}\right)^2}\\
    &+\left(1+\frac{1}{2\UD_1^2}\right)\sum_{i=m+1}^K\left(\frac{\delta}{\Ndelta}\right)^{\frac{1}{4}\left(\frac{\UD_i}{\max\{\theta_U,1-\theta_L\}}\right)^2}\Biggr)
    \end{aligned}
    \end{align*}
\end{theorem}
\begin{proof}
See Appendix~\ref{app:th:apt-bc}. 
\hspace*{\fill}$\Box$\end{proof}

The next corollary is easily derived from Theorem~\ref{th:apt-bc}.

\begin{corollary}
  If $m\geq 1$, then 
  \begin{align*}
\lim_{\delta\rightarrow +0}\frac{\E[T]}{\ln\frac{1}{\delta}}\leq\sum_{i=1}^m\frac{\lim_{\delta\rightarrow +0}\P\left[\hat{i}_1=i,\Ev_i^\mathrm{POS}\right]}{2\Delta_i^2}\leq\frac{1}{2\Delta_m^2}
  \end{align*}
holds for the expected stopping time $\E[T]$ of algorithm $\baec{\mathrm{APT}_\mathrm{P}}{\underline{\mu}}{\overline{\mu}}$.
\end{corollary}

\subsection{Comparison with $\baec{\mathrm{UCB}}{\underline{\mu}}{\overline{\mu}}$}\label{sec:CompWithUCB}

HDoC {\footnotesize (Hybrid algorithm for the Dilemma of Confidence)}\citep{KHSMNS2017} for good arm identification problem
uses arm selection policy UCB (Upper Confidence Bound) \citep{ACF2002}, in which 
\[
\mathrm{UCB}(t,i)=\begin{cases}
\infty & (n_i(t)=0)\\
\hat{\mu}_i(n_i(t)) + \sqrt{\frac{1}{2 n_i(t)} \ln{t}} & (n_i(t)>0)
\end{cases}
\]
is used as $\mathrm{ASP}(t,i)$.
In this section, we analyze a sample complexity upper bound of algorithm\footnote{This is not completely the same algorithm as HDoC because,
  in the HDoC's stopping condition, bounds $\hat{\mu}_{i}(n_i(t))\pm \sqrt{\frac{1}{2n_i(t)} \ln{\frac{4Kn_i(t)^2}{\delta}}}$ are used.} $\baec{\mathrm{UCB}}{\underline{\mu}}{\overline{\mu}}$ and compare it with that of $\baec{\mathrm{APT}_\mathrm{P}}{\underline{\mu}}{\overline{\mu}}$.

  Define $\Delta_{1i}$ as $\Delta_{1i}=\mu_1-\mu_i$. Then, we can obtain the following theorem and corollary, from which, 
we know that, when $\delta$ is small, the dominant terms of our upper bound on the expected stopping time of 
algorithm $\baec{\mathrm{UCB}}{\underline{\mu}}{\overline{\mu}}$, are
$\frac{1}{2\Delta_i^2}\ln\frac{1}{\delta}$ ($i:\mu_i=\mu_1$),
whose sum is $\frac{|\{i\mid \mu_i=\mu_1\}|}{2\Delta_1^2}\ln\frac{1}{\delta}$.

\begin{theorem}\label{th:ucb-bc}
    If $m\geq 1$, then expected stopping time $\E[T]$ of algorithm $\baec{\mathrm{UCB}}{\underline{\mu}}{\overline{\mu}}$ is upper bounded as 
  \begin{align*}
    \E[T]\leq &
    \sum_{i:\mu_i=\mu_1}
    \left(\frac{1}{2\Delta_i^2}\ln\frac{K\Ndelta}{\delta}+O\left(\left(\ln\frac{K\Ndelta}{\delta}\right)^{\frac{2}{3}}\right)\right)\\
    & + \sum_{i:\mu_i<\mu_1}\left(\frac{\ln K\Td{\Delta}}{2\Delta_{1i}^2}+O((\ln K\Td{\Delta})^{\frac{2}{3}}) \right)\\&+O((\ln K\Td{\Delta})^{\frac{2}{3}}\ln\ln K\Td{\Delta})+\frac{\e^{2\UD_1^2}K\Td{\Delta}}{2\UD_1^2}\left(\frac{\delta}{\Ndelta}\right)^{\left(\frac{\UD_1}{\max\{\theta_U,1-\theta_L\}}\right)^2}.
  \end{align*}
\end{theorem}
\begin{proof}
  See Appendix~\ref{proof:ucb-bc}.
\hspace*{\fill}$\Box$\end{proof}

\begin{corollary}
  If $m\geq 1$, then 
  \begin{align*}
\lim_{\delta\rightarrow +0}\frac{\E[T]}{\ln\frac{1}{\delta}}\leq\frac{|\{i\mid \mu_i=\mu_1\}|}{2\Delta_i^2}
  \end{align*}
holds for the expected stopping time $\E[T]$ of algorithm $\baec{\mathrm{UCB}}{\underline{\mu}}{\overline{\mu}}$.
\end{corollary}

\begin{remark}
  From the upper bound shown by Ineq.~(\ref{th:est}), inequality
  \[
\lim_{\delta\rightarrow +0}\frac{\E[T]}{\ln\frac{1}{\delta}}\leq\sum_{i=1}^K\frac{1}{2\Delta_i^2}
  \]
  is derived. This means that the expected stopping time upper bounds for algorithm $\baec{\mathrm{APT}_\mathrm{P}}{\underline{\mu}}{\overline{\mu}}$ and $\baec{\mathrm{UCB}}{\underline{\mu}}{\overline{\mu}}$ shown in Theorem~\ref{th:apt-bc} and \ref{th:ucb-bc} are asymptotically smaller than
  that of algorithm $\baec{\ast}{\underline{\mu}}{\overline{\mu}}$ as $\delta\rightarrow +0$.
 \end{remark} 

\begin{remark}
  When all the arms have Bernoulli loss distributions,
  the right-hand side of Ineq.~(\ref{lower-bound-positive}) in Theorem~\ref{thlowerboud} can be upper bounded as
\begin{align*}
  \frac{1-2\delta}{d(\mu_1,\theta_L)}\ln\frac{1-\delta}{\delta}
  \le
  \frac{1-2\delta}{2\Delta_1^2}\ln\frac{1-\delta}{\delta}
\end{align*}
by Pinsker's Inequality.
Considering tightness of Pinsker's Inequality,
$\frac{1}{2\Delta_1^2}$ is considered to be a tight upper bound of $\lim_{\delta\rightarrow +0}\frac{\E[T]}{\ln\frac{1}{\delta}}$ if Ineq.~(\ref{lower-bound-positive}) is tight.
There is a large gap between $\sum_{i=1}^m\frac{\lim_{\delta\rightarrow +0}\P\left[\hat{i}_1=i,\Ev_i^\mathrm{POS}\right]}{2\Delta_i^2}$
and $\frac{1}{2\Delta_1^2}$, and improvement of the upper bound on the number of samples for $\mathrm{APT}_\mathrm{P}$ seems 
difficult, so the algorithm BAEC with arm selection policy $\mathrm{APT}_\mathrm{P}$ does not seem asymptotically optimal
unless $\lim_{\delta\rightarrow +0}\P\left[\hat{i}_1=1,\Ev_i^\mathrm{POS}\right]=1$.
On the other hand, $\lim_{\delta\rightarrow +0}\frac{\E[T]}{\ln\frac{1}{\delta}}$ for $\mathrm{UCB}$ is upper bounded by $\frac{1}{2\Delta_1^2}$, that is, asymptotically optimal when $\mu_i<\mu_1$ for all arm $i \neq 1$. In the case with $\mu_i=\mu_1$ for all $i=1,\dots,m$, however,
$\lim_{\delta\rightarrow +0}\frac{\E[T]}{\ln\frac{1}{\delta}}\leq\frac{m}{2\Delta_1^2}$ holds for $\mathrm{UCB}$ while the corresponding bound for $\mathrm{APT}_\mathrm{P}$ is asymptotically optimal, that is, $\lim_{\delta\rightarrow +0}\frac{\E[T]}{\ln\frac{1}{\delta}}\leq\frac{1}{2\Delta_1^2}$ holds.
\end{remark}

\begin{remark}
  Comparing non-dominant terms of $\baec{\mathrm{APT}_\mathrm{P}}{\underline{\mu}}{\overline{\mu}}$ and $\baec{\mathrm{UCB}}{\underline{\mu}}{\overline{\mu}}$, a cause for the large upper bound of the expected stopping time can be the existence of arms $i$ whose loss mean $\mu_i$ is close to $\mu_1$ in $\baec{\mathrm{UCB}}{\underline{\mu}}{\overline{\mu}}$ while it can be the existence of arms $i$ whose loss mean $\mu_i$ is close to $\theta$ in $\baec{\mathrm{APT}_\mathrm{P}}{\underline{\mu}}{\overline{\mu}}$.
\end{remark}

\section{Experiments}

In this section, we report the results of our experiments
that were conducted in order to demonstrate the effectiveness of our stopping condition
and arm selection policy on the stopping time.

In all the tables of experimental results, the smallest averaged stopping time in each parameter setting is bolded or italic, and bolded ones mean statistically significant difference.

\subsection{Effectiveness of $\Delta$-Dependent Asymmetric Confidence Bounds}\label{exp:stopcond}

As upper and lower confidence bounds $\mathrm{LB}$ and $\mathrm{UB}$,
we proposed $\underline{\mu}$ and $\overline{\mu}$ based on $\Delta$-dependent asymmetric bounds $\UB{i}{n}$ and $\LB{i}{n}$ defined by Eq.~(\ref{lb}), instead of $\underline{\mu}'$ and $\overline{\mu}'$ based on conventional non-$\Delta$-dependent symmetric bounds $\UBd{i}{n}$ and $\LBd{i}{n}$ defined by Eq.~(\ref{lbd}).
In this subsection, we empirically compare the number of draws for an arm with mean $\mu_i$ to satisfy the stopping condition using those bounds. 

In the experiment, an i.i.d. loss sequence $X_i(1),\cdots$ was generated according to a Bernoulli distribution with mean $\mu_i$
and we measured the stopping time $\tau_i$ which is the smallest $n$ that satisfies the stopping condition ($\LB{i}{n}\geq \theta_L$ or $\UB{i}{n}< \theta_U$).
The stopping times were averaged over 100 runs for each combination of parameters $\delta=0.001, 0.01$, $\mu_i = 0.2, 0.4, 0.6, 0.8$ and $(\theta_L,\theta_U) = (0.1,0.3), (0.3,0.5), (0.5,0.7), (0.7,0.9)$. Note that $\Delta=\theta_U-\theta_L=0.2$ for all the setting.
We used $K=100$ so as to make the bounds asymmetric. As a result, $\alpha=1.154,1.186$ for $\delta=0.001, 0.01$, respectively.
So, $\theta$ is $(\theta_L+\theta_U)/2+0.007$ for $\delta=0.001$ and $(\theta_L+\theta_U)/2+0.009$ for $\delta=0.01$.

\begin{table*}[t]
\begin{center}
  \caption{Number of draws of arm $i$ with loss mean $\mu_i$ until stopping condition is satisfied.
    The numbers are averaged over 100 runs and the intervals determined by `$\pm$' with its following numbers are their 99\% confidence intervals.
  }
\label{tbl:3}
\scriptsize
\begin{tabular}{@{\hspace*{1mm}}c@{\hspace*{1mm}}|@{}c@{\hspace*{1mm}}|@{\hspace*{1mm}}c@{\hspace*{1mm}}|@{\hspace*{1mm}}c@{\hspace*{1mm}}c@{\hspace*{1mm}}c@{\hspace*{1mm}}c@{\hspace*{1mm}}c@{\hspace*{1mm}}} \hline
  $\begin{array}{c}\delta\\ (\Td{\Delta})\end{array}$  & $(\theta_L,\theta_U)$ & $\mathrm{LB},\mathrm{UB}$ & $\mu_i = 0.2$ & $\mu_i=0.4$ & $\mu_i=0.6$ & $\mu_i=0.8$ \\ \hline
\multirow{8}{*}{$\begin{array}{c}0.01\\(684)\end{array}$} &\multirow{2}{*}{\hspace*{1mm}$(0.1,0.3)$}    &$\underline{\mu},\overline{\mu}$ & $\mathbf{497.64 \pm 28.47}$& $\mathit{88.87 \pm 8.61}$& $\mathit{35.16 \pm 2.88}$& $16.96 \pm 1.23$\\
    &    &$\underline{\mu}',\overline{\mu}'$ & $957.81 \pm 37.56$& $104.95 \pm 10.33$& $37.04 \pm 3.50$& $\mathit{16.82 \pm 1.28}$\\ \cline{2-7}
    &\multirow{2}{*}{\hspace*{1mm}$(0.3,0.5)$}    &$\underline{\mu},\overline{\mu}$ & $\mathbf{63.24 \pm 4.88}$& $\mathbf{427.10 \pm 32.90}$& $\mathit{86.50 \pm 8.18}$& $\mathit{30.79 \pm 1.91}$\\
    &    &$\underline{\mu}',\overline{\mu}'$ & $106.67 \pm 7.83$& $889.36 \pm 51.33$& $103.44 \pm 9.95$& $32.26 \pm 2.35$\\ \cline{2-7}
&\multirow{2}{*}{\hspace*{1mm}$(0.5,0.7)$}    &$\underline{\mu},\overline{\mu}$ & $\mathbf{23.52 \pm 1.87}$& $\mathbf{63.79 \pm 6.54}$& $\mathbf{435.91 \pm 37.37}$& $\mathbf{91.56 \pm 6.99}$\\
    &    &$\underline{\mu}',\overline{\mu}'$ & $35.01 \pm 2.46$& $105.13 \pm 9.73$& $885.55 \pm 47.61$& $109.35 \pm 7.93$\\ \cline{2-7}
    &\multirow{2}{*}{\hspace*{1mm}$(0.7,0.9)$}    &$\underline{\mu},\overline{\mu}$ & $\mathbf{13.90 \pm 0.95}$& $\mathbf{24.85 \pm 2.32}$& $\mathbf{65.07 \pm 6.50}$& $\mathbf{500.93 \pm 29.47}$\\
&    &$\underline{\mu}',\overline{\mu}'$ & $17.60 \pm 1.36$& $34.86 \pm 3.10$& $106.25 \pm 10.03$& $963.05 \pm 34.02$\\ \hline
\multirow{8}{*}{$\begin{array}{c}0.001\\(808)\end{array}$} &\multirow{2}{*}{\hspace*{1mm}$(0.1,0.3)$}    &$\underline{\mu},\overline{\mu}$ & $\mathbf{595.24 \pm 31.77}$& $\mathbf{102.05 \pm 8.26}$& $\mathit{37.26 \pm 3.11}$& $18.07 \pm 1.19$\\
    &    &$\underline{\mu}',\overline{\mu}'$ & $1072.65 \pm 43.99$& $123.16 \pm 10.42$& $39.68 \pm 3.77$& $\mathit{17.37 \pm 1.31}$\\ \cline{2-7}
    &\multirow{2}{*}{\hspace*{1mm}$(0.3,0.5)$}    &$\underline{\mu},\overline{\mu}$ & $\mathbf{75.92 \pm 5.17}$& $\mathbf{560.31 \pm 34.91}$& $\mathit{100.66 \pm 9.37}$& $\mathit{38.76 \pm 2.50}$\\
    &    &$\underline{\mu}',\overline{\mu}'$ & $123.73 \pm 8.64$& $980.23 \pm 49.95$& $119.85 \pm 11.64$& $41.10 \pm 2.91$\\ \cline{2-7}
    &\multirow{2}{*}{\hspace*{1mm}$(0.5,0.7)$}    &$\underline{\mu},\overline{\mu}$ & $\mathbf{29.43 \pm 2.14}$& $\mathbf{73.87 \pm 6.71}$& $\mathbf{546.24 \pm 37.43}$& $\mathbf{107.93 \pm 7.50}$\\
    &    &$\underline{\mu}',\overline{\mu}'$ & $41.32 \pm 2.51$& $116.51 \pm 9.38$& $969.24 \pm 53.78$& $126.04 \pm 8.71$\\ \cline{2-7}
    &\multirow{2}{*}{\hspace*{1mm}$(0.7,0.9)$}    &$\underline{\mu},\overline{\mu}$ & $\mathbf{15.50 \pm 1.05}$& $\mathbf{29.33 \pm 2.50}$& $\mathbf{76.96 \pm 7.08}$& $\mathbf{599.91 \pm 29.82}$\\
    &    &$\underline{\mu}',\overline{\mu}'$ & $19.62 \pm 1.33$& $40.21 \pm 3.31$& $117.49 \pm 10.16$& $1075.36 \pm 39.92$\\ \hline
\end{tabular}
\end{center}
\end{table*}

The result is shown in Table \ref{tbl:3}.
As we can see from the table, the stopping condition using $\Delta$-dependent asymmetric bounds makes the stopping time fast compared to the stopping condition using conventional bounds. The effect of the proposed stopping condition becomes significant when the arm is neutral or negative, notably, 
1.74$\sim$2.08 times faster when $\mu_i\approx\theta$. The reason why effect for negative arms is larger than that for positive arms, is due to the asymmetry of the upper and lower confidence intervals:
the upper confidence interval is smaller than the lower confidence interval.

\subsection{Effectiveness of Arm Selection Policy $\mathrm{APT}_\mathrm{P}$}

\subsubsection{Simulation Using Synthetic Distribution Parameters}

In this experiment, we first generated distribution means $\mu_1,\dots,\mu_{100}$ of 100 arms,
and then ran algorithm $\baec{\mathrm{APT}_\mathrm{P}}{\underline{\mu}}{\overline{\mu}}$ simulating an arm-$i$ draw by generating a loss according to
a Bernoulli distribution with mean $\mu_i$.

For given natural number $m$ and a threshold pair $(\theta_L,\theta_U)$, $m$ distribution means were generated according to a uniform distribution over $[\theta, 1]$ and $100-m$ distribution means were generated according to a uniform distribution over $[0,\theta)$, where $\theta=\theta_U-\frac{1}{1+\alpha}\Delta$.

  For each set of 100 distribution means, we also ran algorithms $\baec{\mathrm{LUCB}}{\underline{\mu}}{\overline{\mu}}$ and
$\baec{\mathrm{UCB}}{\underline{\mu}}{\overline{\mu}}$ in addition to $\baec{\mathrm{APT}_\mathrm{P}}{\underline{\mu}}{\overline{\mu}}$ by generating the same
i.i.d. loss sequence for the same arm, which can be realized by feeding a same seed to a random number generator for the same arm.
Here, arm selection policy LUCB uses
\[
\mathrm{LUCB}(t,i)=\begin{cases}
\infty & (n_i(t)=0)\\
\hat{\mu}_i(n_i(t)) & (n_i(t)>0, t \text{ is odd})\\
\hat{\mu}_i(n_i(t)) + \sqrt{\frac{1}{2 n_i(t)} \ln{\frac{5 K t^4}{4 \delta}}} & (n_i(t)>0, t \text{ is even}).
\end{cases}
\]
Note that LUCB\footnote{LUCB means that both of LCB(lower confidence bound) and UCB(upper confidence bound) are used in the algorithm. In fact,
  it chooses the arm $i$ with the smallest LCB among the arms with the largest $m$ sample means when $m\geq 2$.} \citep{KTAS2012} is an algorithm for the best $k$ arm identification problem, and the above policy is exactly the same arm-selection policy as original LUCB for $k=1$.

For each $m=0, 1, 25, 50, 100$,
we generated 100 sets of 100 distribution means, and ran the three algorithms for each set and
for each combination of parameters $\delta=0.01, 0.001$ and $(\theta_L,\theta_U) = (0.19,0.21), (0.49,0.51), (0.79,0.81),$ $(0.1,0.3), (0.4,0.6),(0.7,0.9)$.
As for threshold pairs $(\theta_L,\theta_U)$, $\Delta=0.02$ for the first three and $\Delta=0.2$ for the last three. 
Stopping times were averaged over 100 runs.

\begin{table*}[t]
\begin{center}
  \caption{The average stopping times $\times 10^{-3}$ of three algorithms, and their 99\% confidence intervals in the simulations using synthetic distribution parameters. 
  }\label{exp:syn}
\scriptsize
\begin{tabular}{@{}c@{\hspace*{1mm}}|@{\hspace*{1mm}}l@{\hspace*{1mm}}|@{\hspace*{1mm}}r@{\hspace*{1mm}}r@{\hspace*{1mm}}r@{\hspace*{1mm}}r@{\hspace*{1mm}}r@{\hspace*{1mm}}} \hline
 $\begin{array}{c}(\theta_L,\theta_U)\\\theta\end{array}$ & Policy & $m = 0$ & $m = 1$ & $m = 25$ & $m = 50$ & $m = 100$ \\ \hline
 \multicolumn{7}{l}{$\Delta=0.2, \delta=0.01 \ (K\Td{\Delta}=68.4)$}\\ \hline
 \multirow{3}{*}{$\begin{array}{c}\hspace*{1mm}(0.1,0.3)\\0.2085\end{array}$}    & $\mathrm{APT}_\mathrm{P}$        & $\mathit{11.65 \pm 1.69}$& $\mathit{0.92 \pm 0.49}$& $\mathbf{0.05 \pm 0.02}$& $\mathbf{0.04 \pm 0.02}$& $\mathbf{0.04 \pm 0.02}$\\
 & LUCB       & $12.16 \pm 1.55$& $0.99 \pm 0.56$& $0.13 \pm 0.00$& $0.14 \pm 0.01$& $0.14 \pm 0.01$\\
 & UCB        & $15.44 \pm 0.59$& $2.43 \pm 1.12$& $0.28 \pm 0.00$& $0.34 \pm 0.00$& $0.46 \pm 0.01$\\ \hline
 \multirow{3}{*}{$\begin{array}{c}\hspace*{1mm}(0.4,0.6)\\0.5085\end{array}$}    & $\mathrm{APT}_\mathrm{P}$      & $\mathit{5.88 \pm 0.81}$& $\mathit{0.98 \pm 0.48}$& $\mathbf{0.06 \pm 0.01}$& $\mathbf{0.07 \pm 0.02}$& $\mathbf{0.05 \pm 0.01}$\\
       & LUCB       & $6.30 \pm 0.72$& $1.13 \pm 0.51$& $0.21 \pm 0.02$& $0.21 \pm 0.02$& $0.21 \pm 0.02$\\
       & UCB        & $7.16 \pm 0.45$& $1.58 \pm 0.57$& $0.50 \pm 0.01$& $0.66 \pm 0.01$& $1.06 \pm 0.02$\\ \hline
\multirow{3}{*}{$\begin{array}{c}\hspace*{1mm}(0.7,0.9)\\0.8085\end{array}$}    & $\mathrm{APT}_\mathrm{P}$ & $\mathit{5.10 \pm 0.39}$& $\mathit{1.17 \pm 0.43}$& $\mathbf{0.19 \pm 0.02}$& $\mathbf{0.16 \pm 0.02}$& $\mathbf{0.18 \pm 0.03}$\\
       & LUCB       & $5.13 \pm 0.36$& $1.40 \pm 0.43$& $0.53 \pm 0.05$& $0.60 \pm 0.07$& $0.62 \pm 0.09$\\
       & UCB        & $5.40 \pm 0.27$& $2.04 \pm 0.43$& $1.82 \pm 0.04$& $2.76 \pm 0.05$& $4.58 \pm 0.07$\\ \hline
 \multicolumn{7}{l}{$\Delta=0.2, \delta=0.001 \ (K\Td{\Delta}=80.8)$}\\ \hline
\multirow{3}{*}{$\begin{array}{c}\hspace*{1mm}(0.1,0.3)\\0.2072\end{array}$}    & $\mathrm{APT}_\mathrm{P}$        & $\mathit{13.44 \pm 2.09}$& $\mathit{1.16 \pm 0.67}$& $\mathbf{0.04 \pm 0.01}$& $\mathbf{0.05 \pm 0.02}$& $\mathbf{0.04 \pm 0.01}$\\
& LUCB       & $14.12 \pm 1.95$& $1.36 \pm 0.81$& $0.15 \pm 0.01$& $0.15 \pm 0.01$& $0.16 \pm 0.01$\\
& UCB        & $17.89 \pm 0.93$& $2.65 \pm 1.27$& $0.30 \pm 0.01$& $0.37 \pm 0.01$& $0.51 \pm 0.01$\\ \hline
\multirow{3}{*}{$\begin{array}{c}\hspace*{1mm}(0.4,0.6)\\0.5072\end{array}$}    & $\mathrm{APT}_\mathrm{P}$        & $\mathit{7.81 \pm 0.92}$& $\mathit{1.09 \pm 0.57}$& $\mathbf{0.09 \pm 0.02}$& $\mathbf{0.09 \pm 0.02}$& $\mathbf{0.08 \pm 0.02}$\\
       & LUCB       & $8.31 \pm 0.80$& $1.25 \pm 0.65$& $0.22 \pm 0.02$& $0.24 \pm 0.02$& $0.24 \pm 0.02$\\
       & UCB        & $8.90 \pm 0.59$& $1.71 \pm 0.64$& $0.54 \pm 0.02$& $0.74 \pm 0.01$& $1.16 \pm 0.02$\\ \hline
\multirow{3}{*}{$\begin{array}{c}\hspace*{1mm}(0.7,0.9)\\0.8072\end{array}$}    & $\mathrm{APT}_\mathrm{P}$        & $\mathit{6.15 \pm 0.58}$& $\mathit{1.33 \pm 0.55}$& $\mathbf{0.21 \pm 0.03}$& $\mathbf{0.22 \pm 0.03}$& $\mathbf{0.19 \pm 0.03}$\\
       & LUCB       & $6.34 \pm 0.50$& $1.70 \pm 0.58$& $0.60 \pm 0.05$& $0.66 \pm 0.08$& $0.80 \pm 0.11$\\
       & UCB        & $6.63 \pm 0.38$& $2.36 \pm 0.54$& $1.99 \pm 0.05$& $3.01 \pm 0.05$& $5.19 \pm 0.08$\\ \hline
 \multicolumn{7}{l}{$\Delta=0.02, \delta=0.01 \ (K\Td{\Delta}=9309.9)$}\\ \hline
 \multirow{3}{*}{$\begin{array}{c}\hspace*{1mm}(0.19,0.21)\\0.2006\end{array}$}    & $\mathrm{APT}_\mathrm{P}$        & $\mathit{338.20 \pm 25.53}$& $8.52 \pm 6.39$& $\mathit{0.13 \pm 0.13}$& $\mathit{0.12 \pm 0.06}$& $\mathbf{0.06 \pm 0.02}$\\
 & LUCB       & $341.02 \pm 24.09$& $\mathit{6.30 \pm 4.30}$& $0.17 \pm 0.01$& $0.18 \pm 0.01$& $0.17 \pm 0.01$\\
 & UCB        & $344.31 \pm 23.44$& $9.35 \pm 5.64$& $0.33 \pm 0.01$& $0.42 \pm 0.01$& $0.60 \pm 0.01$\\ \hline
 \multirow{3}{*}{$\begin{array}{c}\hspace*{1mm}(0.49,0.51)\\0.5006\end{array}$}    & $\mathrm{APT}_\mathrm{P}$        & $\mathit{133.93 \pm 11.34}$& $9.29 \pm 6.92$& $\mathit{0.27 \pm 0.21}$& $\mathit{0.18 \pm 0.11}$& $\mathit{0.27 \pm 0.32}$\\
       & LUCB       & $135.27 \pm 11.25$& $\mathit{8.99 \pm 5.87}$& $0.28 \pm 0.02$& $0.31 \pm 0.03$& $0.34 \pm 0.04$\\
 & UCB        & $135.04 \pm 11.23$& $9.48 \pm 6.03$& $0.69 \pm 0.02$& $0.98 \pm 0.02$& $1.56 \pm 0.04$\\ \hline
 \multirow{3}{*}{$\begin{array}{c}\hspace*{1mm}(0.79,0.81)\\0.8006\end{array}$}    & $\mathrm{APT}_\mathrm{P}$        & $\mathit{84.85 \pm 8.80}$& $\mathit{12.97 \pm 7.09}$& $1.27 \pm 0.91$& $\mathbf{0.71 \pm 0.15}$& $\mathit{1.13 \pm 0.44}$\\
 & LUCB       & $85.55 \pm 8.80$& $17.98 \pm 10.60$& $\mathit{1.05 \pm 0.08}$& $1.29 \pm 0.13$& $1.57 \pm 0.20$\\
        & UCB        & $85.26 \pm 8.78$& $17.14 \pm 9.47$& $3.62 \pm 0.11$& $5.58 \pm 0.11$& $9.91 \pm 0.17$\\ \hline
 \multicolumn{7}{l}{$\Delta=0.02, \delta=0.001 \ (K\Td{\Delta}=10530.7)$}\\ \hline
\multirow{3}{*}{$\begin{array}{c}\hspace*{1mm}(0.19,0.21)\\0.2005\end{array}$}    & $\mathrm{APT}_\mathrm{P}$        & $\mathit{393.86 \pm 26.92}$& $8.83 \pm 6.18$& $\mathbf{0.09 \pm 0.04}$& $0.20 \pm 0.35$& $\mathbf{0.07 \pm 0.02}$\\
       & LUCB       & $397.49 \pm 26.08$& $\mathit{5.17 \pm 3.57}$& $0.17 \pm 0.01$& $\mathit{0.18 \pm 0.01}$& $0.19 \pm 0.02$\\
       & UCB        & $395.85 \pm 26.05$& $9.47 \pm 5.99$& $0.34 \pm 0.01$& $0.44 \pm 0.01$& $0.63 \pm 0.02$\\ \hline
\multirow{3}{*}{$\begin{array}{c}\hspace*{1mm}(0.49,0.51)\\0.5005\end{array}$}    & $\mathrm{APT}_\mathrm{P}$        & $\mathit{154.09 \pm 13.34}$& $\mathit{10.12 \pm 6.72}$& $\mathit{0.28 \pm 0.13}$& $\mathit{0.26 \pm 0.17}$& $\mathit{0.24 \pm 0.19}$\\
       & LUCB       & $158.83 \pm 13.18$& $10.86 \pm 7.43$& $0.30 \pm 0.02$& $0.33 \pm 0.03$& $0.37 \pm 0.04$\\
& UCB        & $158.06 \pm 13.03$& $10.31 \pm 6.52$& $0.77 \pm 0.02$& $1.05 \pm 0.02$& $1.72 \pm 0.04$\\ \hline
\multirow{3}{*}{$\begin{array}{c}\hspace*{1mm}(0.79,0.81)\\0.8005\end{array}$}    & $\mathrm{APT}_\mathrm{P}$        & $\mathit{96.72 \pm 9.47}$& $\mathit{14.99 \pm 7.66}$& $1.23 \pm 0.57$& $1.84 \pm 1.45$& $\mathit{1.23 \pm 0.60}$\\
       & LUCB       & $97.94 \pm 9.69$& $21.21 \pm 12.38$& $\mathit{1.21 \pm 0.09}$& $\mathit{1.46 \pm 0.13}$& $1.80 \pm 0.22$\\
       & UCB & $97.45 \pm 9.53$& $17.80 \pm 9.44$& $3.86 \pm 0.11$& $6.10 \pm 0.13$& $10.61 \pm 0.20$\\ \hline
\end{tabular}
\end{center}
\end{table*}

The result is shown in Table~\ref{exp:syn}.
In the case with large $\Delta(=0.2)$, the stopping time for $\mathrm{APT}_\mathrm{P}$ is the smallest for almost all combinations of parameters in this experiment.
In the case with small $\Delta(=0.02)$, $\baec{\mathrm{APT}_\mathrm{P}}{\underline{\mu}}{\overline{\mu}}$ also stopped first, on average, for most combinations of parameters,
but the stopping time for $\mathrm{LUCB}$ is the smallest for some combinations of parameters when $1\leq m\leq 50$.
Difference between stopping times for $\mathrm{APT}_\mathrm{P}$ and for the other two policies, becomes larger as $m$ increases.
$\baec{\mathrm{APT}_\mathrm{P}}{\underline{\mu}}{\overline{\mu}}$ stopped first even when $m=0$, that is, in the case that all the loss means are below $\theta$.
In such case, some gray zone arms can be judged as positive and make the algorithm stop.
$\baec{\mathrm{APT}_\mathrm{P}}{\underline{\mu}}{\overline{\mu}}$ is considered to have found such gray zone arms faster.

\subsubsection{Simulation Based on Real Dataset}

In this experiment, as loss distribution means, we used estimated ad click rates by users in the same category calculated from Real-Time Bidding dataset provided by iPinYou \citep{zhang2014real}. 
From the training dataset of the second season of iPinYou dataset, we chose 20 most frequently appeared user categories (sets of user profile ids) and calculated the click rate by the users in the category for each of them using the impression and click logs. 
Since the click rates are smaller than 0.001, we used the values multiplied by 100 as loss means. 
The loss means $\mu_1,\dots,\mu_{20}$ used in the experiment are followings:
{\scriptsize
  \begin{center}
\begin{tabular}{@{}l@{ }l@{ }l@{ }l@{ }l@{}} 
  $\mu_1$:$0.06232$, & $\mu_5$:$0.04124$, & $\mu_9$:$0.03792$, & $\mu_{13}$:$0.02535$, & $\mu_{17}$:$0.02183$,\\
  $\mu_2$:$0.05549$, & $\mu_6$:$0.04060$, & $\mu_{10}$:$0.03764$, & $\mu_{14}$:$0.02498$, & $\mu_{18}$:$0.02055$,\\
  $\mu_3$:$0.05011$, & $\mu_7$:$0.04031$, & $\mu_{11}$:$0.03054$, & $\mu_{15}$:$0.02203$, & $\mu_{19}$:$0.01255$,\\
  $\mu_4$:$0.04587$, & $\mu_8$:$0.03907$, & $\mu_{12}$:$0.02594$, & $\mu_{16}$:$0.02197$, & $\mu_{20}$:$0.01033$.
\end{tabular}
\end{center}
}
In this experiment, 5 thresholds $(\theta_L,\theta_U)=(\theta_{m'}-0.01,\theta_{m'}+0.01)$ for $m'=0,1,5,10,19$ are used
so as to let the loss means of about $m'$ arms be at least $\theta$, where
$\theta_0=\mu_1+\frac{\mu_1-\mu_2}{2}$, $\theta_{m'}=\frac{\mu_{m'}+\mu_{m'+1}}{2}$ for $m'=1,5,10,19$.
For these $(\theta_L,\theta_U)$s, $\theta=0.06649,0.05966,0.04168,0.03485,0.01220$ when $\delta=0.001$, and $\theta=0.06659,0.05976,0.04178,0.03495,0.01230$ when $\delta=0.01$. For these $\theta$s, the number of arms whose loss mean is at least $\theta$ is $0, 1, 4, 10, 19$. 
For each combination of parameters $\delta = 0.01, 0.001$, $(\theta_L,\theta_U)=(\theta_{m'}-0.01,\theta_{m'}+0.01)$ ($m'=0,1,5,10,19$),
ran algorithm $\baec{\mathrm{ASP}}{\underline{\mu}}{\overline{\mu}}$ with three arm selection policies $\mathrm{ASP}=\mathrm{APT}_\mathrm{P}$, $\mathrm{LUCB}$ and $\mathrm{UCB}$ 100 times and calculated their stopping times averaged over the 100 runs.

\begin{table*}[t]
\begin{center}
  \caption{The average stopping times $\times 10^{-3}$ of the three algorithms and their 99\% confidence intervals in the simulations based on real dataset. Note that $(\theta_L,\theta_U)=(\theta_{m'}-0.01,\theta_{m'}+0.01)$ for $m'=0,1,5,10,19$. 
  }
\label{tbl:click}\scriptsize
\begin{tabular}{@{\hspace*{1mm}}c@{\hspace*{1mm}}|@{\hspace*{1mm}}l@{\hspace*{1mm}}|@{\hspace*{1mm}}c@{\hspace*{1mm}}c@{\hspace*{1mm}}c@{\hspace*{1mm}}c@{\hspace*{1mm}}c@{\hspace*{1mm}}c@{\hspace*{1mm}}} \hline
  \multirow{2}{*}{$\begin{array}{c}\delta\\(K\Td{\Delta})\end{array}$} & \multirow{2}{*}{Policy} & $\theta_0 = 0.06573$ & $\theta_1 = 0.05890$ & $\theta_5 = 0.04092$ & $\theta_{10} = 0.03409$ & $\theta_{19} = 0.01144$ \\
    & &  $(m=0)$ & $(m=1)$ & $(m=4)$ & $(m=10)$ & $(m=19)$\\ \hline
  \multirow{3}{*}{$\begin{array}{c}0.01\\(1776.1)\end{array}$}   & $\mathrm{APT}_\mathrm{P}$& $\mathit{150.78 \pm 2.68}$ & $\mathbf{62.15 \pm 3.79}$ & $30.83 \pm 6.05$ & $23.85 \pm 3.77$ & $9.26 \pm 1.81$ \\
   & LUCB       & $\mathit{150.78 \pm 2.68}$ & $122.93 \pm 8.07$ & $\mathit{28.89 \pm 3.03}$ & $\mathbf{17.16 \pm 1.26}$ & $\mathit{8.44 \pm 0.83}$ \\
   & UCB        & $\mathit{150.78 \pm 2.68}$ & $149.07 \pm 5.47$ & $51.73 \pm 2.73$ & $38.41 \pm 1.98$ & $20.93 \pm 1.28$ \\ \hline
  \multirow{3}{*}{$\begin{array}{c}0.001\\(1790.8)\end{array}$}   & $\mathrm{APT}_\mathrm{P}$        & $\mathit{174.78 \pm 2.41}$ & $\mathbf{66.11 \pm 3.56}$ & $33.28 \pm 7.51$ & $23.61 \pm 4.24$ & $9.78 \pm 1.84$ \\
  & LUCB       & $\mathit{174.78 \pm 2.41}$ & $129.42 \pm 6.07$ & $\mathit{29.40 \pm 2.43}$ & $\mathit{21.22 \pm 1.55}$ & $\mathit{9.36 \pm 0.81}$ \\
  & UCB& $\mathit{174.78 \pm 2.41}$ & $159.19 \pm 5.76$ & $57.13 \pm 2.68$ & $44.66 \pm 2.16$ & $22.90 \pm 1.23$ \\ \hline
\end{tabular}
\end{center}
\end{table*}

The result is shown in Table~\ref{tbl:click}.
For $m=1$, the stopping times for $\mathrm{APT}_\mathrm{P}$ are significantly small compared with the other two arm selection policies.
For $m=4,10,19$, $\baec{\mathrm{LUCB}}{\underline{\mu}}{\overline{\mu}}$ always stops first though $\baec{\mathrm{APT}_\mathrm{P}}{\underline{\mu}}{\overline{\mu}}$'s stopping times are comparable to those of $\baec{\mathrm{LUCB}}{\underline{\mu}}{\overline{\mu}}$ except the case with $\delta=0.001, m=10$.
When $m = 0$, the stopping times of the three algorithms are equal, which means that all the arms including the unique neutral arm $\mu_1$ were always judged as negative arms in the experiment. 


\section{Conclusions}

We theoretically and empirically studied sample complexity of a \emph{bad arm existence checking problem} (BAEC problem), 
whose objective is to detect existence of some bad arm (arm with loss mean larger than $\theta_U$)
with probability at least $1-\delta$ for given thresholds $\theta_L$ and $\theta_U$ with $\theta_U-\theta_L=\Delta$.
We proposed algorithm $\baec{\mathrm{APT}_\mathrm{P}}{\underline{\mu}}{\overline{\mu}}$ that utilizes asymmetry of positive and negative arms' roles in this problem;
the algorithm with a \emph{stopping condition} for drawing each arm $i$ with the current number of draws $n$ using $\Delta$-dependent asymmetric confidence bounds $\LB{i}{n}$ and $\UB{i}{n}$, and arm selection policy $\mathrm{APT}_\mathrm{P}$ that uses a single threshold $\theta$ closer to $\theta_U$ instead of the center between $\theta_L$ and $\theta_U$.
Effectiveness of our stopping condition was shown empirically and theoretically.
Algorithm $\baec{\mathrm{APT}_\mathrm{P}}{\underline{\mu}}{\overline{\mu}}$
empirically stopped faster than algorithms $\baec{\mathrm{LUCB}}{\underline{\mu}}{\overline{\mu}}$ and $\baec{\mathrm{UCB}}{\underline{\mu}}{\overline{\mu}}$ using conventional arm selection policies LUCB and UCB,
and we showed an asymptotic upper bound of the expected stopping time for $\baec{\mathrm{APT}_\mathrm{P}}{\underline{\mu}}{\overline{\mu}}$  which is smaller than that for $\baec{\mathrm{UCB}}{\underline{\mu}}{\overline{\mu}}$ in the case that there are multiple positive arms and all the positive arms have the same loss means. 
Current theoretical support for our arm selection policy $\mathrm{APT}_\mathrm{P}$ is very limited,
and further theoretical analysis that explains its empirically observed small stopping times is our future work.

\begin{acknowledgements}
This work was partially supported by JST CREST Grant Number JPMJCR1662, Japan.
\end{acknowledgements}

\bibliographystyle{spbasic}      

\bibliography{ecmlmlj2019}

\appendix

\section{Proof of Theorem~\ref{thlowerboud}}\label{appendix:theorem1}

We use the following lemma to prove our lower bound on the number of samples needed for a $(\Delta,\delta)$-BAEC algorithm.

\begin{lemma}[\citealt{KCG2016}]\label{kaufmann}
  Let ${\bm\nu}$ and ${\bm \nu}'$ be two loss distribution sets of $K$ arms
  such that distributions $\nu_i$ and $\nu'_i$ are mutually absolutely continuous for $i=1,\dots,K$.
  For any almost-surely finite stopping time $T$ and any event $\mathcal{E}$, the following inequality holds.
  \[
\sum_{i=1}^K\E_{\bm \nu}[n_i(T)]\mathrm{KL}(\nu_i,\nu'_i)\geq d(\mathbb{P}_{\bm \nu}(\mathcal{E}),\mathbb{P}_{\bm \nu'}(\mathcal{E})).
  \]
\end{lemma}

\noindent
{\it Proof of Theorem~\ref{thlowerboud}.}\ 
  Consider a set ${\bm \nu}$ of Bernoulli distributions $\nu_i$ with mean $\mu_i$ for which some positive arms exist, that is, the case with $\mu_1\geq\theta_U$.
  Let $k$ be the number of arms $i$ with $\mu_i\geq \theta_L$ in $\{\nu_i\}$, that means
  $\mu_1\geq\cdots\geq \mu_k\geq \theta_L > \mu_{k+1}\geq \cdots \geq \mu_K$.
  For an arbitrary fixed $\epsilon>0$, let $\{\nu'_i\}$ be the set of Bernoulli distributions with means $\mu'_i$ defined as
  \[
  \mu'_i=\left\{\begin{array}{ll}
  \theta_L-\epsilon & (i\leq k)\\
  \mu_i & (i>k)
  \end{array}\right.
  \]
For any $(\Delta,\delta)$-BAEC algorithm,
$\Ev_\mathrm{POS}$ denotes the event that its output is ``positive''.
Since some positive arms exist for the distribution set ${\bm \nu}$,
the probability that the event $\Ev_\mathrm{POS}$ occurs must be at least $1-\delta$ by Definition~\ref{def:delta-delta-alg}, that is,
inequality $\mathbb{P}_{\bm \nu}(\Ev_\mathrm{POS})\geq 1-\delta$ holds.
All the arms are negative in the distribution set ${\bm \nu'}=\{\nu'_i\}$,
likewise by Definition~\ref{def:delta-delta-alg}, inequality $\mathbb{P}_{\bm \nu'}(\Ev_\mathrm{POS})<\delta$ holds.
Thus,
\begin{align*}
 \sum_{i=1}^K\E[n_i(T)]KL(\nu_i,\nu'_i)
  =&\sum_{i=1}^k\E[n_i(T)]d(\mu_i,\mu'_i)\ \ (\text{by } d(\mu_i,\mu_i)=0)\\
    =& \sum_{i=1}^k\E[n_i(T)]d(\mu_i,\theta_L-\epsilon)\\
    \geq & d(\mathbb{P}_{\bm \nu}(\Ev_\mathrm{POS}),\mathbb{P}_{\bm \nu'}(\Ev_\mathrm{POS}))\ \ (\text{by Lemma~\ref{kaufmann}})\\
    >&  d(1-\delta,\delta)
  \end{align*}
holds.
From the fact that $\max_{i\in \{1,\dots,k\}}d(\mu_i,\theta_L-\epsilon)=d(\mu_1,\theta_L-\epsilon)$,
  \begin{align*}
    \E[T]=\sum_{i=1}^K\E[n_i(T)]> & \frac{d(1-\delta,\delta)}{d(\mu_1,\theta_L-\epsilon)}
    = \frac{1-2\delta}{d(\mu_1,\theta_L-\epsilon)}\ln\frac{1-\delta}{\delta}
  \end{align*}
  holds, which leads to Ineq. (\ref{lower-bound-positive}) by considering its limit as $\epsilon\rightarrow +0$.
  
Next, consider a set ${\bm \nu}$ of Bernoulli distributions $\nu_i$ with mean $\mu_i$ for which all the arms are negative, that is, the case with $\mu_1<\theta_L$.
Fix  $j\in \{1,\dots,K\}$ arbitrarily.
For arbitrary $\epsilon>0$, let ${\bm \nu'}$ be a set of 
Bernoulli distributions $\nu'_i$ with mean $\mu'_i$ defined as
  \[
  \mu'_i=\left\{\begin{array}{ll}
  \theta_U+\epsilon & (i=j)\\
  \mu_i & (i\neq j)
  \end{array}\right.
  \]

  For any $(\Delta,\delta)$-BAEC algorithm, $\Ev_\mathrm{NEG}$ denotes the event that
  its output is ``negative''.
  Then, inequalities
  $\mathbb{P}_{\bm \nu}(\Ev_\mathrm{NEG})\geq 1-\delta$ and $\mathbb{P}_{\bm \nu'}(\Ev_\mathrm{NEG})<\delta$ hold by Definition~\ref{def:delta-delta-alg} because all the arms are negative in $\bm \nu$ and arm $j$ is positive in $\bm \nu'$.
Thus, by Lemma~\ref{kaufmann},
  \begin{align*}
    \E[n_j(T)]d(\mu_j,\theta_U+\epsilon)\geq & d(\mathbb{P}_{\bm \nu}(\Ev_\mathrm{NEG}),\mathbb{P}_{\bm \nu'}(\Ev_\mathrm{NEG}))
    >  d(1-\delta,\delta)
  \end{align*}
 holds, that is, for each $j=1,\dots,K$,
  \begin{align*}
    \E[n_j(T)]> &\frac{d(1-\delta,\delta)}{d(\mu_j,\theta_U+\epsilon)}
    = \frac{1-2\delta}{d(\mu_j,\theta_U+\epsilon)}\ln\frac{1-\delta}{\delta}
  \end{align*}
  holds. This leads to Ineq.~(\ref{lower-bound-negative}) by considering
  its limit as $\epsilon\rightarrow +0$ and the summation over $j=1,\dots,K$.
\hspace*{\fill}$\Box$

\section{Proof of Theorem~\ref{lemUtBjbW9v}}\label{appendix:lemma1}

We prove Theorem~\ref{lemUtBjbW9v} using the following proposition.

\begin{proposition}\label{prop:1}
For any $x > 0$, the following inequality holds:
\[
\sqrt{4 + x} \le \sqrt{1 + x} + 1 \le \sqrt{4 + 2 x}.
\]
\end{proposition}
\begin{proof}
Since 
\[
\sqrt{1 + x} + 1 = \sqrt{(\sqrt{1 + x} + 1)^2}
= \sqrt{2 + x + 2 \sqrt{1 + x}}
\]
holds,
\[
\sqrt{4 + x}=\sqrt{2 + x + 2} \le \sqrt{1 + x} + 1
\]
and
\[
\sqrt{4 + 2x}=\sqrt{2 + x + 2 \left(1 + \frac{x}{2}\right)} \ge \sqrt{1 + x} + 1
\]
hold for $x>0$.
\hspace*{\fill}$\Box$\end{proof}

\noindent
{\it Proof of Theorem~\ref{lemUtBjbW9v}.}\ 
 We prove this theorem by contradiction.
Assume that $\overline{\mu}_i(\Td{\Delta})\geq \theta_U$ and $\theta_L > \underline{\mu}_i(\Td{\Delta})$.
Then,
  \begin{align}
  \overline{\mu}_i(\Td{\Delta}) - \underline{\mu}_i(\Td{\Delta}) > \theta_U - \theta_L = \Delta \label{uldiff}
\end{align}
holds. 
On the other hand, 
 \begin{align*}
  \Delta =& \sqrt{\frac{2}{\frac{2}{\Delta^2}\ln\frac{\sqrt{K}\Ndelta}{\delta}}\ln\frac{\sqrt{K}\Ndelta}{\delta}} \\
   \geq & \sqrt{\frac{4}{2\Td{\Delta}} \ln{\frac{\sqrt{K} \Ndelta}{\delta}}} \\
   = &  \sqrt{\frac{1}{2\Td{\Delta}} \ln{\frac{\Ndelta}{\delta}}} \sqrt{4 + \frac{2\ln{K}}{\ln{\frac{\Ndelta}{\delta}}}} \\
   \ge & \sqrt{\frac{1}{2\Td{\Delta}} \ln{\frac{\Ndelta}{\delta}}} \left( \sqrt{1 + \frac{\ln{K}}{\ln{\frac{\Ndelta}{\delta}}}} + 1 \right) \ \ \text{ (by Proposition~\ref{prop:1})}\\
   = & \sqrt{\frac{1}{2\Td{\Delta}} \ln{\frac{K \Ndelta}{\delta}}} + \sqrt{\frac{1}{2\Td{\Delta}} \ln{\frac{\Ndelta}{\delta}}}=\overline{\mu}_i(\Td{\Delta}) - \underline{\mu}_i(\Td{\Delta})
 \end{align*} 
 
 holds, which contradicts to Ineq.~(\ref{uldiff}).
\hspace*{\fill}$\Box$

\section{Proof of Theorem\ref{th:remark1}}\label{appendix:remark1}

  If $\UBd{i}{\Tdd{\Delta}}-\LBd{i}{\Tdd{\Delta}}>\Delta$ holds, then $\UBd{i}{n}-\LBd{i}{n}>\Delta$ holds for $n=1,\dots,\Tdd{\Delta}$. In this case, $\LBd{i}{n}<\theta_L$ and  $\UBd{i}{n}\geq \theta_U$ hold for $n=1,\dots,\Tdd{\Delta}$ when $\theta_U- (\UBd{i}{n}-\LBd{i}{n})/2\leq \hat{\mu}_i(n)<\theta_L+ (\UBd{i}{n}-\LBd{i}{n})/2$, which means $\tau'_i>\Tdd{\Delta}$.
In fact, Inequality $\UBd{i}{\Tdd{\Delta}}-\LBd{i}{\Tdd{\Delta}}>\Delta$ holds because
  \begin{align*}
    &\UBd{i}{\Tdd{\Delta}}-\LBd{i}{\Tdd{\Delta}}=2\sqrt{\frac{1}{2\Tdd{\Delta}}\ln\frac{2K{\Tdd{\Delta}}^2}{\delta}}\\
=& 2\sqrt{\frac{1}{2\lfloor\frac{2}{\Delta^2}\ln\frac{448K}{\Delta^4\delta}\rfloor}\ln\frac{2K\lfloor\frac{2}{\Delta^2}\ln\frac{448K}{\Delta^4\delta}\rfloor^2}{\delta}}\\
\geq &2\sqrt{\frac{1}{2\frac{2}{\Delta^2}\ln\frac{448K}{\Delta^4\delta}}\ln\frac{2K\left(\frac{2}{\Delta^2}\ln\frac{448K}{\Delta^4\delta}\right)^2}{\delta}}\ \ \ \left(\text{because } f(x)=\frac{\ln x}{x} \text{ is decreasing for } x\geq \e\right)\\
= &\Delta\sqrt{\frac{1}{\ln\frac{448K}{\Delta^4\delta}}\ln\left(\frac{8K}{\Delta^4\delta}\left(\ln\frac{448K}{\Delta^4\delta}\right)^2\right)}\\
> &\Delta\sqrt{\frac{1}{\ln\frac{448K}{\Delta^4\delta}}\ln\left(\frac{8K}{\Delta^4\delta}\cdot 56\right)}\ \ \ \left(\text{by $\left(\ln\frac{448K}{\Delta^4\delta}\right)^2\!>\! \left(\ln\frac{448\cdot 2}{1^4(\frac{1}{2})}\right)^2\!=\!56.11\cdots>56$}\right)\\
= &\Delta\sqrt{\frac{1}{\ln\frac{448K}{\Delta^4\delta}}\ln\frac{448K}{\Delta^4\delta}}=\Delta.
\end{align*}

The difference between the worst case stopping times $\tau'_i-\tau_i$ is lower-bounded as
\begin{align*}
  \tau'_i-\tau_i> &\Tdd{\Delta}-\Td{\Delta}=\left\lfloor\frac{2}{\Delta^2}\ln\frac{448K}{\Delta^4\delta}\right\rfloor-\left\lceil\frac{2}{\Delta^2}\ln\frac{\sqrt{K}\Ndelta}{\delta}\right\rceil\\
  >&\frac{2}{\Delta^2}\ln\frac{448K}{\Delta^4\delta}-\frac{2}{\Delta^2}\ln\frac{\sqrt{K}\Ndelta}{\delta}-2\\
  =&\frac{2}{\Delta^2}\ln\frac{448\sqrt{K}}{\Delta^4\Ndelta\e^{\Delta^2}}\\
  > & \frac{2}{\Delta^2}\ln\frac{448\sqrt{K}e^{-\Delta^2}}{\Delta^4\left(\frac{2\e}{(\e-1)\Delta^2}\ln\frac{2\sqrt{K}}{\Delta^2\delta}+1\right)}\\
  = & \frac{2}{\Delta^2}\ln\frac{448\sqrt{K}e^{-\Delta^2}\cdot\frac{(\e-1)}{2\e\Delta^2}}{\ln\frac{2\sqrt{K}}{\Delta^2\delta}+\frac{(\e-1)\Delta^2}{2\e}}\\
  = & \frac{2}{\Delta^2}\ln\frac{\frac{224\sqrt{K}e^{-\Delta^2-1}(\e-1)}{\Delta^2}}{\ln\frac{2\sqrt{K}}{\Delta^2\delta}\e^{\frac{(\e-1)\Delta^2}{2\e}}}\\
  >& \frac{2}{\Delta^2}\ln\frac{224\sqrt{K}e^{-2}(\e-1)/\Delta^2}{\ln\frac{2\sqrt{K}}{\Delta^2\delta}\e^{\frac{\e-1}{2\e}}}\ \ \left(\text{by $\Delta< 1$}\right)\\
  >  & \frac{2}{\Delta^2}\ln\frac{52\sqrt{K}/\Delta^2}{\ln\frac{3\sqrt{K}}{\Delta^2\delta}}\ \ \left(\text{by $224\e^{-2}(\e-1)>52$ and $2\e^{\frac{\e-1}{2\e}}<3$}\right)\\
  = & \frac{2}{\Delta^2}\left(\ln\frac{52\sqrt{K}}{\Delta^2}-\ln\ln\frac{3\sqrt{K}}{\Delta^2\delta}\right).
\end{align*}
\hspace*{\fill}$\Box$

\section{Proof of Theorem~\ref{thwelldef}}\label{app:thwelldef}

Proposition~\ref{prop:2} is used in the proof of Lemma~\ref{lem:E*}.
The following proposition is needed to prove Proposition~\ref{prop:2}.

\begin{proposition}\label{ineqZNBcoVb1}
  For $0<a<1$, any $t \geq \frac{\e}{(\e-1)a} \ln\frac{1}{a}$ satisfies the following inequality. 
  \[
  a t \ge \ln t.
  \]
\end{proposition}

\begin{proof} For $0<a<1$, let $f(t) = a t - \ln{t}$. 
  When $a > \frac{1}{\e}$, $f(t)$ is always positive for any $t > 0$ since $f(t)$ takes minimum value $1 - \ln{\frac{1}{a}}$ at $t = \frac{1}{a}$. 
  
  When $a \le \frac{1}{\e}$, if $t = \frac{\e}{(\e-1)a} \ln\frac{1}{a}$, 
  \begin{align*}
    a t -\ln t = \left(\frac{1}{\e - 1} \ln{\frac{1}{a}} - \ln{\frac{\e}{\e-1}}\right) - \ln{\ln{\frac{1}{a}}} \ge 0
    \end{align*}
holds because $y=\frac{1}{\e - 1}x - \ln\frac{\e}{\e-1}$ is a tangential line of $y=\ln x$ at $x=\e-1$.
  If $t > \frac{\e}{(\e-1)a} \ln\frac{1}{a} \left(\ge \frac{\e}{(\e-1)a} > \frac{1}{a}\right)$, $\frac{d f(t)}{d t} = a - \frac{1}{t}$ is positive. 
  Therefore, for $t \ge \frac{\e}{(\e-1)a} \ln\frac{1}{a}$, $a t -\ln t \ge 0$.
\hspace*{\fill}$\Box$\end{proof}

\begin{proposition}\label{prop:2}
  $\Td{\Delta}\leq \Ndelta$.
\end{proposition}
\begin{proof}
The following inequality is derived from Proposition~\ref{ineqZNBcoVb1} by setting $a$ to $\frac{\Delta^2\delta}{2\sqrt{K}}$ that means $t = \frac{\sqrt{K}\Ndelta}{\delta} \ge \frac{2\e\sqrt{K}}{(\e-1)\Delta^2\delta} \ln\frac{2\sqrt{K}}{\Delta^2\delta}$, 

\begin{align*}
  \ln\frac{\sqrt{K}\Ndelta}{\delta} \le \frac{\Delta^2 \delta}{2\sqrt{K}} \cdot\frac{\sqrt{K}\Ndelta}{\delta}=\frac{\Delta^2\Ndelta}{2}.
\end{align*}
Thus,
\begin{align*}
\Ndelta\geq \frac{2}{\Delta^2}\ln\frac{\sqrt{K}\Ndelta}{\delta}
 \end{align*}
holds, and so
\begin{align*}
\Ndelta\geq \left\lceil\frac{2}{\Delta^2}\ln\frac{\sqrt{K}\Ndelta}{\delta}\right\rceil=\Td{\Delta}
 \end{align*}
holds.
  
\hspace*{\fill}$\Box$\end{proof}

\begin{lemma}\label{lem:E*}
  For the complementary events $\overline{\Ev^+}$, $\overline{\Ev^-}$ of events $\Ev^+$, $\Ev^-$,
  inequality $\P\{\overline{\Ev^+}\}\leq \delta$ holds when $\mu_1\geq\theta_U$ and inequality $\P\{\overline{\Ev^-}\}\leq \delta$ holds when $\mu_1<\theta_L$.
\end{lemma}
\begin{proof}
Assume that $\mu_1\geq \theta_U$.
  Using De Morgan's laws, $\overline{\Ev^+}$ can be expressed as
  \begin{align*}
    \overline{\Ev^+}=& \bigcap_{i:\mu_i\geq \theta_U}\bigcup_{n=1}^{\Td{\Delta}}\Bigl\{\UB{i}{n}< \mu_i\Bigr\}\\
    = & \bigcap_{i:\mu_i\geq \theta_U}\bigcup_{n=1}^{\Td{\Delta}}\left\{\hat{\mu}_i(n)<\mu_i-\sqrt{\frac{1}{2n}\ln\frac{\Ndelta}{\delta}}\right\}.
    \end{align*}
  So, the probability that event $\overline{\Ev^+}$ occurs is bounded by $\delta$ using Hoeffding's Inequality:
  \begin{align*}
    \P\{\overline{\Ev^+}\}\leq&\max_{i:\mu_i\geq \theta_U}\sum_{n=1}^{\Td{\Delta}}\P\left\{\hat{\mu}_i(n)<\mu_i-\sqrt{\frac{1}{2n}\ln\frac{\Ndelta}{\delta}}\right\}\\
\leq &\sum_{n=1}^{\Td{\Delta}}\frac{\delta}{\Ndelta}=\frac{\Td{\Delta}}{\Ndelta}\delta\leq \delta.
\ \ ( \text{by Proposition~\ref{prop:2}})
  \end{align*}

Assume that $\mu_1< \theta_L$.
  Using De Morgan's laws, $\overline{\Ev^-}$ can be expressed as
  \begin{align*}
    \overline{\Ev^-}=&  \bigcup_{i=1}^K\bigcup_{n=1}^{\Td{\Delta}}\Bigl\{\LB{i}{n}\geq \mu_i\Bigr\}\\
    = & \bigcup_{i=1}^K\bigcup_{n=1}^{\Td{\Delta}}\left\{\hat{\mu}_i(n)\geq \mu_i+\sqrt{\frac{1}{2n}\ln\frac{K\Ndelta}{\delta}}\right\}.
    \end{align*}
  So, the probability that event $\overline{\Ev^-}$ occurs is bounded by $\delta$ using the union bound and Hoeffding's Inequality:
  \begin{align*}
    \P\{\overline{\Ev^-}\}\leq&\sum_{i=1}^K\sum_{n=1}^{\Td{\Delta}}\P\left\{\hat{\mu}_i(n)\geq\mu_i+\sqrt{\frac{1}{2n}\ln\frac{K\Ndelta}{\delta}}\right\}\\
\leq &\sum_{i=1}^K\sum_{n=1}^{\Td{\Delta}}\frac{\delta}{K\Ndelta}=\frac{\Td{\Delta}}{\Ndelta}\delta\leq \delta.
\ \ ( \text{by Proposition~\ref{prop:2}})
  \end{align*}

\hspace*{\fill}$\Box$\end{proof}  

\noindent
{\it Proof of Theorem~\ref{thwelldef}.}\ 
By the definition of $\tau_i$, algorithm~$\baec{\ast}{\underline{\mu}}{\overline{\mu}}$ draws arm $i$ at most $\tau_i$ times,
which is upper-bounded by $\Td{\Delta}$ due to Lemma~\ref{lemUtBjbW9v}.
So, algorithm~$\baec{\ast}{\underline{\mu}}{\overline{\mu}}$ stops after at most $K\Td{\Delta}$ arm draws.

When at least one arm is positive, that is, in the case with $\mu_1\geq \theta_U$,
algorithm~$\baec{\ast}{\underline{\mu}}{\overline{\mu}}$ returns ``positive'' if event $\Ev^+$ occurs.
Thus, algorithm~$\baec{\ast}{\underline{\mu}}{\overline{\mu}}$ returns ``positive'' with probability $\P\{\Ev^+\}=1-\P\{\overline{\Ev^+}\}\geq 1-\delta$ by Lemma~\ref{lem:E*}.
When all the arms are negative, that is, in the case with $\mu_1< \theta_L$,
algorithm~$\baec{\ast}{\underline{\mu}}{\overline{\mu}}$ returns ``negative'' if event $\Ev^-$ occurs.
Thus, algorithm~$\baec{\ast}{\underline{\mu}}{\overline{\mu}}$ returns ``negative'' with probability $\P\{\Ev^-\}=1-\P\{\overline{\Ev^-}\}\geq 1-\delta$ by Lemma~\ref{lem:E*}.
\hspace*{\fill}$\Box$

\section{Proof of Theorem~\ref{theorem3}}\label{appendix:theorem3}

Consider the case that $\mu_1\geq \theta_U$ and event $\Ev^+$ occurs.
In this case,  $\bigcap_{n=1}^{\Td{\Delta}}\{\UB{i}{n}\geq \mu_i\}$ holds for some $i$ with $\mu_i\geq \theta_U$.
Assume $\Td{\Delta_i}<\tau_i$ for this $i$. Then, $\UB{i}{\Td{\Delta_i}}\geq \mu_i\geq \theta_U$ and $\LB{i}{\Td{\Delta_i}}< \theta_L$ hold. However,
  \begin{align*}
    \LB{i}{\Td{\Delta_i}}
    =& \hat{\mu}_i-\sqrt{\frac{1}{2\Td{\Delta_i}}\ln\frac{K\Ndelta}{\delta}}\\
    \geq & \mu_i -\sqrt{\frac{1}{2\Td{\Delta_i}}\ln\frac{\Ndelta}{\delta}} -\sqrt{\frac{1}{2\Td{\Delta_i}}\ln\frac{K\Ndelta}{\delta}} \ \ \text{ (by $\UB{i}{\Td{\Delta_i}}\geq \mu_i$)}\\
    = &\mu_i -\sqrt{\frac{1}{2\Td{\Delta_i}}\ln\frac{\Ndelta}{\delta}}\left(\sqrt{1+\frac{\ln K}{\ln\frac{\Ndelta}{\delta}}} +1\right)\\
    \geq & \mu_i- \sqrt{\frac{1}{2\Td{\Delta_i}}\ln\frac{\Ndelta}{\delta}}\sqrt{4+\frac{2\ln K}{\ln\frac{\Ndelta}{\delta}}}\ \ \ \text{ (by Proposition~\ref{prop:1})}\\
    =& \mu_i -\sqrt{\frac{4}{2\Td{\Delta_i}}\ln \frac{\sqrt{K}\Ndelta}{\delta}}
    \geq  \mu_i -\Delta_i = \theta_L\ \ \  \left(\text{by } \Td{\Delta_i}\geq \frac{2}{\Delta_i^2}\ln \frac{\sqrt{K}\Ndelta}{\delta}\right)\\
  \end{align*}
  holds, which contradicts the fact that $\LB{i}{\Td{\Delta_i}}< \theta_L$.
   Thus, $\tau_i\leq \Td{\Delta_i}$ holds for at least one positive arm $i$ with probability $\P\{\Ev^+\}$ which is at least $1-\delta$ by Lemma~\ref{lem:E*}. 

   Consider the case that $\mu_1< \theta_L$ holds and event $\Ev^-$ occurs. 
Assume $\Td{\Delta_i}<\tau_i$ for $i=1,\dots,K$. Then, $\UB{i}{\Td{\Delta_i}}\geq \theta_U$ and $\LB{i}{\Td{\Delta_i}}< \mu_i <\theta_L$ hold. However,
  \begin{align*}
    \UB{i}{\Td{\Delta_i}}
    =& \hat{\mu}_i+\sqrt{\frac{1}{2\Td{\Delta_i}}\ln\frac{\Ndelta}{\delta}}\\
    < & \mu_i +\sqrt{\frac{1}{2\Td{\Delta_i}}\ln\frac{\Ndelta}{\delta}} +\sqrt{\frac{1}{2\Td{\Delta_i}}\ln\frac{K\Ndelta}{\delta}} \ \ \text{ (by $\LB{i}{\Td{\Delta_i}}< \mu_i$)}\\
    = &\mu_i +\sqrt{\frac{1}{2\Td{\Delta_i}}\ln\frac{\Ndelta}{\delta}}\left(\sqrt{1+\frac{\ln K}{\ln\frac{\Ndelta}{\delta}}} +1\right)\\
    \leq & \mu_i+ \sqrt{\frac{1}{2\Td{\Delta_i}}\ln\frac{\Ndelta}{\delta}}\sqrt{4+\frac{2\ln K}{\ln\frac{\Ndelta}{\delta}}}\ \ \ \text{ (by Proposition~\ref{prop:1})}\\
    =& \mu_i +\sqrt{\frac{4}{2\Td{\Delta_i}}\ln \frac{\sqrt{K}\Ndelta}{\delta}}
    \leq  \mu_i +\Delta_i = \theta_U\ \ \  \left(\text{by } \Td{\Delta_i}\geq \frac{2}{\Delta_i^2}\ln \frac{\sqrt{K}\Ndelta}{\delta}\right)\\
  \end{align*}
    holds, which contradicts the fact that $\UB{i}{\Td{\Delta_i}}\geq \theta_U$.
 Thus, $\tau_i\leq \Td{\Delta_i}$ holds for all arms $i$ with probability $\P\{\Ev^-\}$ which is at least $1-\delta$ by Lemma~\ref{lem:E*}. 
  
\hspace*{\fill}$\Box$

\section{Proof of Lemma~\ref{lem:expectedtau}}\label{app:lem:expectedtau}

  Let $\epsilon$ be an arbitrary real that satisfies $0<\epsilon<\Delta/2(1+\alpha)$. 

  Consider the case with $\mu_i\geq \theta$. Define $n_i$ as $n_i=\frac{1}{2(\Delta_i-\epsilon)^2}\ln\frac{K\Ndelta}{\delta}$.
  Then,
  \begin{align*}
    \E[\tau_i\mathbb{1}\{\Ev\}] = & \sum_{n=1}^{\infty}n\P[\tau_i=n,\Ev]=\sum_{n=1}^{\infty}\P[\tau_i\geq n,\Ev]\\
    \leq & \sum_{n=2}^{\infty}\P[\LB{i}{n-1}<\theta_L,\Ev]+1\\
    = & \sum_{n=1}^{\infty}\P[\LB{i}{n}<\theta_L,\Ev]+1\\
    \leq  & \sum_{n=1}^{\lfloor n_i\rfloor}\P[\Ev]+\sum_{n=\lfloor n_i\rfloor+1}^{\infty}\P[\LB{i}{n}<\theta_L]+1\\
    & \Bigl(\text{because}
           \P[\LB{i}{n}<\theta_L,\Ev]\leq \min\{\P[\LB{i}{n}<\theta_L],\P[\Ev]\}\Bigr)\\
    \leq  & \P[\Ev]n_i+\sum_{n=\lfloor n_i\rfloor+1}^{\infty}\P\left[\hat{\mu}_i(n)-\sqrt{\frac{1}{2n_i}\ln\frac{K\Ndelta}{\delta}}<\theta_L\right]+1\\
    &  \left(\text{because $\sqrt{\frac{1}{2n_i}\ln\frac{K\Ndelta}{\delta}}\geq \sqrt{\frac{1}{2n}\ln\frac{K\Ndelta}{\delta}}$ for $n\geq n_i$}\right)\\    
    =  & \P[\Ev]n_i+\sum_{n=\lfloor n_i\rfloor+1}^{\infty}\P[\hat{\mu}_i(n)<\mu_i-\epsilon]+1\ \ \  \left(\text{because $\sqrt{\frac{1}{2n_i}\ln\frac{K\Ndelta}{\delta}}=\Delta_i-\epsilon$}\right)\\
    \leq  & \P[\Ev]n_i+\sum_{n=\lfloor n_i\rfloor+1}^{\infty}\e^{-2n\epsilon^2}+1  \ \ \ (\text{by Hoeffding's Inequality})\\
    \leq & \P[\Ev]n_i+\frac{1}{\e^{2\epsilon^2}-1}+1\leq\P[\Ev]n_i+\frac{1}{2\epsilon^2}+1 \ \ \ (\text{because $\e^x-1\geq x$ for any real $x$})\\
    = & \frac{\P[\Ev]}{2(\Delta_i-\epsilon)^2}\ln\frac{K\Ndelta}{\delta}+\frac{1}{2\epsilon^2}+1
  \end{align*}
  holds. Since $\frac{1}{\Delta_i^2}+\frac{6\epsilon}{\Delta_i^3}- \frac{1}{(\Delta_i-\epsilon)^2}=\frac{\epsilon(\Delta_i-2\epsilon)(4\Delta_i-3\epsilon)}{\Delta_i^3(\Delta_i-\epsilon)^2}\geq 0$ holds for $0<\epsilon\leq \frac{\Delta_i}{2}$,
$\frac{1}{(\Delta_i-\epsilon)^2}\leq \frac{1}{\Delta_i^2}+\frac{6\epsilon}{\Delta_i^3}$ holds for $0<\epsilon<\Delta/2(1+\alpha)\leq \Delta_i/2$. Thus,  Ineq.~(\ref{eq:lem1}) can be obtained by setting $\epsilon$ to  $O((\ln\frac{K\Ndelta}{\delta})^{-1/3})$.

  Next, consider the case with $\mu_i<\theta$. Define $n_i$ as $n_i=\frac{1}{2(\Delta_i-\epsilon)^2}\ln\frac{\Ndelta}{\delta}$.
  Then,
  \begin{align*}
    \E[\tau_i\mathbb{1}\{\Ev\}] = & \sum_{n=1}^{\infty}n\P[\tau_i=n,\Ev]=\sum_{n=1}^{\infty}\P[\tau_i\geq n,\Ev]\\
    \leq & \sum_{n=2}^{\infty}\P[\UB{i}{n-1}\geq\theta_U,\Ev]+1\\
    = & \sum_{n=1}^{\infty}\P[\UB{i}{n}\geq\theta_U,\Ev]+1
  \end{align*}
  holds. Similar calculation leads to Inequality~(\ref{eq:lem1(2)}).
\hspace*{\fill}$\Box$

\section{Proof of Theorem~\ref{th:apt-bc}}\label{app:th:apt-bc}

Define $\mathrm{apt}_\mathrm{P}(n,i)$ as $\mathrm{apt}_\mathrm{P}(n,i)=\sqrt{n}(\hat{\mu}_i(n)-\theta)$ for convenience.
Note that $\mathrm{APT}_\mathrm{P}(t,i)=\mathrm{apt}_\mathrm{P}(n_i(t),i)$.
Random variables $Y_i$ and $N_i(a)$ are defined as
\begin{align*}
  Y_i=& \min_{n\in\{1,\dots,\tau_i\}}\mathrm{apt}_\mathrm{P}(n,i) \text{ and }\\
  N_i(a) = & \min \left(\{n|n\in \{1,\dots,\tau_i-1\},\mathrm{apt}_\mathrm{P}(n,i)<a\}\cup\{\tau_i\}\right).
\end{align*}

To obtain an upper bound of the expected stopping time $\E[T]$ for algorithm $\baec{\mathrm{APT}_\mathrm{P}}{\underline{\mu}}{\overline{\mu}}$,
we consider the case that, for some arm $i$ with $\mu_i\geq \theta$, arm $i$ is the first arm that satisfies stopping condition and $\underline{\mu}_i(\tau_i)\geq \theta_L$, that is, the case that event $\{\hat{i}_1=i, \Ev_i^\mathrm{POS}\}$ occurs.
In the case with no such arm $i$, stopping time $T$ is upper bounded by the worst case bound $K\Td{\Delta}$ (Theorem~\ref{thwelldef}) and the decreasing order of the occurrence probability of this case as $\delta\rightarrow +0$ can be proved to be small compared to the increasing order of $K\Td{\Delta}$ (Lemma~\ref{lem:m1im+1E*} and \ref{lem:CompEventProb}), so it can be ignored asymptotically as $\delta\rightarrow +0$.
An upper bound of $\E[T\mathbb{1}\{\hat{i}_1=i, \Ev_i^\mathrm{POS}\}]$ for arm $i$ with $\mu_i\geq \theta$ is proved in Lemma~\ref{lem:i1iE*}.
When event $\{\hat{i}_1=i, \Ev_i^\mathrm{POS}\}$ occurs for arm $i$ with $\mu_i\geq \theta$,
the number of arm draws is $\tau_i$ for arm $i$, at most $N_j(Y_i)$ for arm $j\neq i$ if $Y_i\leq 0$ and at most $N_j(0)$ for arm $j\neq i$ if $Y_i>0$.
So, to prove the upper bound in Lemma~\ref{lem:i1iE*}, we upper bound $\E[\tau_i\mathbb{1}\{\hat{i}_1=i, \Ev_i^\mathrm{POS}\}]$ by Lemma~\ref{lem:expectedtau},
$\E[N_j(Y_i)\mathbb{1}\{Y_i\leq 0,\hat{i}_1=i, \Ev_i^\mathrm{POS}\}]$ for $j\neq i$ by Lemma~\ref{lem:imjm} and \ref{lem:imjm+1}
and $\E[N_j(0)\mathbb{1}\{Y_i>0,\hat{i}_1=i, \Ev_i^\mathrm{POS}\}]$ for $j\neq i$ by Lemma~\ref{lem:imj}.

\begin{lemma}\label{lem:posnja} $\baec{\mathrm{APT}_\mathrm{P}}{\underline{\mu}}{\overline{\mu}}$ satisfies
  \begin{align*}
    \sum_{n=1}^{\infty}\P[N_j(a)\geq n]< \frac{2}{\UD_j^4} \text{\ \ for } j\leq m \text{ and } a\leq 0.
  \end{align*}
\end{lemma}
\begin{proof}
  \begin{align*}
     \sum_{n=1}^{\infty}\P[N_j(a)\geq n]
    = & \sum_{n=1}^{\infty}\sum_{t=n}^{\infty}\P[N_j(a)=t]\\
    \leq & \sum_{n=1}^{\infty}\sum_{t=n}^{\infty}\P[\mathrm{apt}_\mathrm{P}(t,j)<a]\leq \sum_{n=1}^{\infty}\sum_{t=n}^{\infty}\P[\mathrm{apt}_\mathrm{P}(t,j)<0]\\
    =& \sum_{n=1}^{\infty}\sum_{t=n}^{\infty}\P[\sqrt{t}(\hat{\mu}_j(t)-\theta)<0]\\
    =& \sum_{n=1}^{\infty}\sum_{t=n}^{\infty}\P[\hat{\mu}_j(t)<\mu_j-\UD_j]\\
    \leq & \sum_{n=1}^{\infty}\sum_{t=n}^{\infty}\e^{-2t\UD_j^2} \ \ (\text{by Hoeffding's Inequality})\\
    =& \sum_{n=1}^{\infty}\frac{\e^{-2n\UD_j^2}}{1-\e^{-2\UD_j^2}}=\frac{\e^{-2\UD_j^2}}{(1-\e^{-2\UD_j^2})^2}=\frac{\e^{2\UD_j^2}}{(\e^{2\UD_j^2}-1)^2}\\
    <& \frac{\e^2}{4\UD_j^4}< \frac{2}{\UD_j^4} \ \ (\text{because $\UD_j<1$})
  \end{align*}
  
\hspace*{\fill}$\Box$\end{proof}

\begin{lemma}\label{lem:imjm}
$\baec{\mathrm{APT}_\mathrm{P}}{\underline{\mu}}{\overline{\mu}}$ satisfies
  \begin{align*}
    \E[N_j(Y_i)\mathbb{1}\{Y_i\leq 0\}]\leq \frac{2}{\UD_j^4}\P[Y_i\leq 0] 
    \end{align*}
for $i=1,\dots,K$ and $j\leq m \ (i\neq j)$.
\end{lemma}
\begin{proof}
  Define $\mathbb{F}_i(a)$ as $\mathbb{F}_i(a)=\P[Y_i\leq a]$. Then, 
  \begin{align*} 
    \E[N_j(Y_i)\mathbb{1}\{Y_i\leq 0\}] =& \sum_{n=1}^{\infty}n\P[N_j(Y_i)=n,Y_i\leq 0]\\
    =& \sum_{n=1}^{\infty}\P[N_j(Y_i)\geq n,Y_i\leq 0]\\
    =& \int_{-\infty}^0\sum_{n=1}^{\infty}\P[N_j(Y_i)\geq n \mid Y_i=a]\d\mathbb{F}_i(a)\\
    = & \int_{-\infty}^0\sum_{n=1}^{\infty}\P[N_j(a)\geq n]\d\mathbb{F}_i(a)\\
    \leq &\frac{2}{\UD_j^4} \int_{-\infty}^0\d\mathbb{F}_i(a) \ \ (\text{by Lemma~\ref{lem:posnja}})\\
    = &\frac{2}{\UD_j^4} [\P[Y_i\leq a]]_{-\infty}^0=\frac{2}{\UD_j^4}\P[Y_i\leq 0]
  \end{align*}
  holds.
  
\hspace*{\fill}$\Box$\end{proof}

\begin{lemma}\label{lem:apybcjp} $\baec{\mathrm{APT}_\mathrm{P}}{\underline{\mu}}{\overline{\mu}}$ satisfies
  \begin{align*}
    \sum_{n=1}^{\infty}\P[N_j(a)\geq n] < \frac{4a^2}{\UD_j^2}+\frac{4}{\UD_j^2}+1
    \end{align*}
for  $j\geq m+1$  and  $a\leq 0$. 
\end{lemma}
\begin{proof}
  Define $n_0$ as $n_0=\frac{4a^2}{\UD_j^2}$. Note that $\UD_j+\frac{a}{\sqrt{n}}> \frac{\UD_j}{2}$ for $n> n_0$.
  Then, 
   \begin{align*}
    \sum_{n=1}^{\infty}\P[N_j(a)\geq n] \leq & \sum_{n=1}^{\infty}\P[\mathrm{apt}_\mathrm{P}(n-1,j)\geq a]
    \leq \sum_{n=1}^{\infty}\P[\mathrm{apt}_\mathrm{P}(n,j)\geq a]+1\\=&\sum_{n=1}^{\infty}\P[\sqrt{n}(\hat{\mu}_j(n)-\theta)\geq a]+1\\
    = & \sum_{n=1}^{\infty}\P\left[\hat{\mu}_j(n)\geq \theta+\frac{a}{\sqrt{n}}\right]+1\\
    = &\sum_{n=1}^{\infty}\P\left[\hat{\mu}_j(n)\geq \mu_j+\UD_j+\frac{a}{\sqrt{n}}\right]+1\\
    \leq& \sum_{n=1}^{\lfloor n_0\rfloor}1+\sum_{n=\lfloor n_0\rfloor+1}^{\infty}\P\left[\hat{\mu}_j(n)\geq \mu_j+\UD_j+\frac{a}{\sqrt{n}}\right]+1\\
    \leq&n_0+\sum_{n=\lfloor n_0\rfloor+1}^{\infty}\e^{-2n\left(\frac{\UD_j}{2}\right)^2}+1\\
    &\!\!\!\!\!\!\!\!\! \Bigl(\text{by Hoeffding's Inequality and the fact that $\UD_j+\frac{a}{\sqrt{n}}>\frac{\UD_j}{2}$ for $n>n_0$}\Bigr)\\
    \leq &\frac{4a^2}{\UD_j^2}+\frac{\e^{-n_0\frac{\UD_j^2}{2}}}{1-\e^{-\frac{\UD_j^2}{2}}}+1=\frac{4a^2}{\UD_j^2}+\frac{\e^{\frac{\UD_j^2}{2}}}{\e^{\frac{\UD_j^2}{2}}-1}\e^{-2a^2}+1\\
    \leq &\frac{4a^2}{\UD_j^2}+\frac{2\e^{\frac{\UD_j^2}{2}}}{\UD_j^2}+1\leq \frac{4a^2}{\UD_j^2}+\frac{2\e^{\frac{1}{2}}}{\UD_j^2}+1< \frac{4a^2}{\UD_j^2}+\frac{4}{\UD_j^2}+1
   \end{align*}
   
\hspace*{\fill}$\Box$\end{proof}

\begin{lemma}\label{lem:yia}
$\baec{\mathrm{APT}_\mathrm{P}}{\underline{\mu}}{\overline{\mu}}$ satisfies
  \begin{align*}
    \P[Y_i\leq a]\leq \frac{\e^{-2a^2}}{2\UD_i^2} \text{\ \ for } i\leq m \text{ and } a\leq 0.
    \end{align*}
\end{lemma}
\begin{proof}
  \begin{align*}
    \P[Y_i\leq a] \leq & \P\left[\bigcup_{n=1}^{\infty} \{\mathrm{apt}_\mathrm{P}(n,i)\leq a\}\right]\\
    \leq & \sum_{n=1}^{\infty}\P[\mathrm{apt}_\mathrm{P}(n,i)\leq a]\\
    = & \sum_{n=1}^{\infty}\P[\sqrt{n}(\hat{\mu}_i(n)-\theta)\leq a]\\
    = & \sum_{n=1}^{\infty}\P\left[\hat{\mu}_i(n)\leq \theta+\frac{a}{\sqrt{n}}\right]\\
    = & \sum_{n=1}^{\infty}\P\left[\hat{\mu}_i(n)\leq \mu_i-\UD_i+\frac{a}{\sqrt{n}}\right]\\
    \leq & \sum_{n=1}^{\infty} \e^{-2n\left(\UD_i-\frac{a}{\sqrt{n}}\right)^2}\\
    \leq & \e^{-2a^2}\sum_{n=1}^{\infty} \e^{-2n\UD_i^2}=\e^{-2a^2}\frac{1}{\e^{2\UD_i^2}-1}\leq \frac{\e^{-2a^2}}{2\UD_i^2}
  \end{align*}
  
\hspace*{\fill}$\Box$\end{proof}

\begin{lemma}\label{lem:imjm+1}
  For $i\leq m$ and $j\geq m+1$, $\baec{\mathrm{APT}_\mathrm{P}}{\underline{\mu}}{\overline{\mu}}$ satisfies
  \begin{align*}
    \E[N_j(Y_i)\mathbb{1}\{Y_i\leq 0\}]\leq \frac{1}{\UD_i^2\UD_j^2}+\left(\frac{4}{\UD_j^2}+1\right)\P[Y_i\leq 0].
    \end{align*}
\end{lemma}
\begin{proof}
    Define $\mathbb{F}_i(a)$ as $\mathbb{F}_i(a)=\P[Y_i\leq a]$. Then, 
  \begin{align*}
    &\E[N_j(Y_i)\mathbb{1}\{Y_i\leq 0\}] = \sum_{n=1}^{\infty}\P[N_j(Y_i)\geq n,Y_i\leq 0]\\
    =& \int_{-\infty}^0\sum_{n=1}^{\infty}\P[N_j(Y_i)\geq n\mid Y_i=a]\d\mathbb{F}_i(a)\\
    = & \int_{-\infty}^0\sum_{n=1}^{\infty}\P[N_j(a)\geq n]\d\mathbb{F}_i(a)\\
    \leq &\int_{-\infty}^0\left(\frac{4a^2}{\UD_j^2}+\frac{4}{\UD_j^2}+1\right)\d\mathbb{F}_i(a) \ \ (\text{by Lemma~\ref{lem:apybcjp}})\\
    = &\frac{4}{\UD_j^2}\int_{-\infty}^0a^2\d\mathbb{F}_i(a)+\left(\frac{4}{\UD_j^2}+1\right)\int_{-\infty}^0\d\mathbb{F}_i(a)\\
    = &\frac{4}{\UD_j^2}\left(\left[a^2\P[Y_i\leq a]\right]_{-\infty}^0-\int_{-\infty}^02a\P[Y_i\leq a]\d a \right) +\left(\frac{4}{\UD_j^2}+1\right)\left[\P[Y_i\leq a]\right]_{-\infty}^0\\
    & \hspace{5cm} (\text{using integration by parts})\\
    = &-\frac{4}{\UD_j^2}\int_{-\infty}^02a\P[Y_i\leq a]\d a +\left(\frac{4}{\UD_j^2}+1\right)\P[Y_i\leq 0]\\
    \leq &-\frac{2}{\UD_i^2\UD_j^2}\int_{-\infty}^02a\e^{-2a^2}\d a +\left(\frac{4}{\UD_j^2}+1\right)\P[Y_i\leq 0]\ \ \ (\text{by Lemma~\ref{lem:yia}})\\
    = &\frac{2}{\UD_i^2\UD_j^2}\left[\frac{e^{-2a^2}}{2}\right]_{-\infty}^0 +\left(\frac{4}{\UD_j^2}+1\right)\P[Y_i\leq 0]\\
    = & \frac{1}{\UD_i^2\UD_j^2}+\left(\frac{4}{\UD_j^2}+1\right)\P[Y_i\leq 0]
  \end{align*}  
  
\hspace*{\fill}$\Box$\end{proof}

\begin{lemma}\label{lem:imj}
  For $i\leq m$, $\baec{\mathrm{APT}_\mathrm{P}}{\underline{\mu}}{\overline{\mu}}$ satisfies
  \begin{align*}
    \E[N_j(0)\mathbb{1}\{Y_i> 0\}]    \leq  \begin{cases}
      \frac{2}{\UD_j^4}\P[Y_i>0] & (j\leq m)\\
      \left(\frac{4}{\UD_j^2}+1\right)\P[Y_i>0] & (j\geq m+1).
      \end{cases}
    \end{align*}
\end{lemma}
\begin{proof}
  \begin{align*}
    \E[N_j(0)\mathbb{1}\{Y_i> 0\}] = & \sum_{n=1}^{\infty}\P[N_j(0)\geq n,Y_i> 0]\\
    = & \sum_{n=1}^{\infty}\P[N_j(0)\geq n]\P[Y_i> 0]\\
    & \ \ \ (\text{because $N_j(0)$ and $Y_i$ are independent})\\
    \leq & \begin{cases}
      \frac{2}{\UD_j^4}\P[Y_i>0] & (j\leq m) \ \ (\text{by Lemma~\ref{lem:posnja}})\\
      \left(\frac{4}{\UD_j^2}+1\right)\P[Y_i>0] & (j\geq m+1) \ \ (\text{by Lemma~\ref{lem:apybcjp}})
      \end{cases}
  \end{align*}
  
\hspace*{\fill}$\Box$\end{proof}

\begin{lemma}\label{lem:i1iE*}
  For $i\leq m$ and any event $\Ev$, $\baec{\mathrm{APT}_\mathrm{P}}{\underline{\mu}}{\overline{\mu}}$ satisfies
  \begin{align*}
    \E[T\mathbb{1}\{\hat{i}_1=i,\Ev\}]
  \leq & 
\begin{aligned}[t]
    & \frac{\P[\hat{i}_1=i,\Ev]}{2\Delta_i^2}\ln\frac{K\Ndelta}{\delta}+O\left(\left(\ln\frac{K\Ndelta}{\delta}\right)^{\frac{2}{3}}\right) \\
    & +\sum_{j\leq m,j\neq i}\frac{2}{\UD_j^4} +\sum_{j=m+1}^K \left\{\frac{1}{\UD_i^2\UD_j^2}+\left(\frac{4}{\UD_j^2}+1\right)\right\}.
\end{aligned}
\end{align*}
\end{lemma}
\begin{proof}
  In the case that
  the stopping condition is satisfied first by one of arms $i$ with $\mu_i\geq \theta$ ($i\leq m$), that is, $\hat{i}_1=i$,
  the stopping time $T$ is at most $\tau_i+\sum_{j\neq i}N_j(Y_i)$ if $Y_i\leq 0$
  and at most $\tau_i+\sum_{j\neq i}N_j(0)$ if $Y_i>0$.
  Thus, for $i\leq m$,
\begin{align*}
  &\E[T\mathbb{1}\{\hat{i}_1=i,\Ev\}]\\
  \leq &
    \E\left[\left(\tau_i+\sum_{j\neq i}N_j(Y_i)\right)\mathbb{1}\{Y_i\leq 0,\hat{i}_1=i,\Ev\}\right]
    +\E\left[\left(\tau_i+\sum_{j\neq i}N_j(0)\right)\mathbb{1}\{Y_i> 0,\hat{i}_1=i,\Ev\}\right]\\
  =& 
  \E[\tau_i\mathbb{1}\{\hat{i}_1=i,\Ev\}]
    +\sum_{j\neq i}\E[N_j(Y_i)\mathbb{1}\{Y_i\leq 0, \hat{i}_1=i,\Ev\}]
  +\sum_{j\neq i}\E[N_j(0)\mathbb{1}\{Y_i> 0, \hat{i}_1=i,\Ev\}]\\
  \leq &
  \E[\tau_i\mathbb{1}\{\hat{i}_1=i,\Ev\}]
    +\sum_{j\neq i}\E[N_j(Y_i)\mathbb{1}\{Y_i\leq 0\}]+\sum_{j\neq i}\E[N_j(0)\mathbb{1}\{Y_i> 0\}]\\
  \leq & \begin{aligned}[t]
    & \frac{\P[\hat{i}_1=i,\Ev]}{2\Delta_i^2}\ln\frac{K\Ndelta}{\delta}+O\left(\left(\ln\frac{K\Ndelta}{\delta}\right)^{\frac{2}{3}}\right)+\sum_{j\leq m,j\neq i}\frac{2}{\UD_j^4}\P[Y_i\leq 0]\ \ \ ( \text{by Lemma~\ref{lem:expectedtau} \& \ref{lem:imjm}})\\
    & +\sum_{j=m+1}^K \left\{\frac{1}{\UD_i^2\UD_j^2}+\left(\frac{4}{\UD_j^2}+1\right)\P[Y_i\leq 0]\right\}\ \ \  ( \text{by Lemma~\ref{lem:imjm+1}})\\
    &+ \sum_{j\leq m,j\neq i}\frac{2}{\UD_j^4}\P[Y_i> 0] +\sum_{j=m+1}^K \left(\frac{4}{\UD_j^2}+1\right)\P[Y_i>0]\ \ \  ( \text{by Lemma~\ref{lem:imj}})
  \end{aligned}\\
  \leq & 
     \frac{\P[\hat{i}_1=i,\Ev]}{2\Delta_i^2}\ln\frac{K\Ndelta}{\delta}+O\left(\left(\ln\frac{K\Ndelta}{\delta}\right)^{\frac{2}{3}}\right) \\
   & +\sum_{j\leq m,j\neq i}\frac{2}{\UD_j^4} +\sum_{j=m+1}^K \left\{\frac{1}{\UD_i^2\UD_j^2}+\left(\frac{4}{\UD_j^2}+1\right)\right\}
\end{align*}
holds.
  
\hspace*{\fill}$\Box$\end{proof}  

Define $\ndelta$ as $\ndelta=\left\lceil\frac{1}{2(\max\{\theta_U,1-\theta_L\})^2}\ln \frac{\Ndelta}{\delta}\right\rceil$.
Then, $\tau_i$ for any arm $i=1,\dots,K$ is bounded by $\ndelta$ from below.

\begin{lemma}\label{lemma:taulb} In algorithm~$\baec{\ast}{\underline{\mu}}{\overline{\mu}}$, 
     $\tau_i\geq \ndelta$ holds for any arm $i=1,\dots,K$.
   \end{lemma}
\begin{proof}
  By the definition of $\tau_i$,
  $\UB{i}{\tau_i}<\theta_U$ or $\LB{i}{\tau_i}\geq \theta_L$ must be satisfied for any arm $i$.
  In the case with $\UB{i}{\tau_i}<\theta_U$,
  \begin{align*}
    \hat{\mu}_i(\tau_i)+\sqrt{\frac{1}{2\tau_i}\ln\frac{\Ndelta}{\delta}} <\theta_U
  \end{align*}
  holds. Since $\hat{\mu}_i(\tau_i)\geq 0$,
    \begin{align*}
    \sqrt{\frac{1}{2\tau_i}\ln\frac{\Ndelta}{\delta}} <\theta_U
  \end{align*}
    holds. So, we obtain
    \begin{align*}
      \tau_i> \frac{1}{2\theta_U^2}\ln\frac{\Ndelta}{\delta}.
    \end{align*}
  In the case with $\LB{i}{\tau_i}\geq \theta_L$,
  \begin{align*}
    \hat{\mu}_i(\tau_i)-\sqrt{\frac{1}{2\tau_i}\ln\frac{K\Ndelta}{\delta}} \geq \theta_L
  \end{align*}
  holds. Since $\hat{\mu}_i(\tau_i)\leq 1$,
    \begin{align*}
    1-\sqrt{\frac{1}{2\tau_i}\ln\frac{K\Ndelta}{\delta}} \geq \theta_L
  \end{align*}
    holds. So, we obtain
    \begin{align*}
      \tau_i\geq \frac{1}{2(1-\theta_L)^2}\ln\frac{K\Ndelta}{\delta}.
    \end{align*}
Therefore,
\begin{align*}
\tau_i\geq & \min\left\{\frac{1}{2\theta_U^2}\ln\frac{\Ndelta}{\delta},\frac{1}{2(1-\theta_L)^2}\ln\frac{K\Ndelta}{\delta}\right\}
\geq  \frac{1}{2(\max\{\theta_U,1-\theta_L\})^2}\ln \frac{\Ndelta}{\delta}
\end{align*}
holds.
Since $\tau_i$ is a natural number,
\[
\tau_i\geq \left\lceil\frac{1}{2(\max\{\theta_U,1-\theta\})^2}\ln\frac{\Ndelta}{\delta}\right\rceil
\]
holds.
\hspace*{\fill}$\Box$\end{proof}

\begin{lemma}\label{lem:yiaE*}
  $\baec{\mathrm{APT}_\mathrm{P}}{\underline{\mu}}{\overline{\mu}}$ satisfies
  \begin{align*}
    \P\left[Y_i\geq -\frac{\UD_i}{2}\sqrt{\ndelta}\right] \leq \e^{-\ndelta\frac{\UD_i^2}{2}}
    \leq \left(\frac{\delta}{\Ndelta}\right)^{\frac{1}{4}\left(\frac{\UD_i}{\max\{\theta_U,1-\theta_L\}}\right)^2} 
    \end{align*}
    for $i\geq m+1$.
  \end{lemma}
\begin{proof}
  \begin{align*}
    \P\left[Y_i\geq -\frac{\UD_i}{2}\sqrt{\ndelta}\right]
    =& \P\left[\bigcap_{n=1}^{\tau_i}\left\{\mathrm{apt}_\mathrm{P}(n,i)\geq -\frac{\UD_i}{2}\sqrt{\ndelta}\right\}\right]\\
    \leq  &\P\left[\mathrm{apt}_\mathrm{P}(\ndelta,i)\geq -\frac{\UD_i}{2}\sqrt{\ndelta}\right]\ \ \text{ (by Lemma~\ref{lemma:taulb})}\\
    =&\P\left[\sqrt{\ndelta}(\hat{\mu}_i(\ndelta)-\theta)\geq -\frac{\UD_i}{2}\sqrt{\ndelta}\right]\\
    = &\P\left[\hat{\mu}_i(\ndelta)\geq \mu_i+\frac{\UD_i}{2}\right]\\
    \leq & \e^{-2\ndelta\left(\frac{\UD_i}{2}\right)^2}=\e^{-\ndelta\frac{\UD_i^2}{2}}\\
    \leq & \e^{-\frac{1}{4}\left(\frac{\UD_i}{\max\{\theta_U,1-\theta_L\}}\right)^2\ln\frac{\Ndelta}{\delta}}=\left(\frac{\delta}{\Ndelta}\right)^{\frac{1}{4}\left(\frac{\UD_i}{\max\{\theta_U,1-\theta_L\}}\right)^2}
  \end{align*}
  
\hspace*{\fill}$\Box$\end{proof}

\begin{lemma}\label{lem:m1im+1E*}
  For $m\geq 1$ and $i\geq m+1$, $\baec{\mathrm{APT}_\mathrm{P}}{\underline{\mu}}{\overline{\mu}}$ satisfies
  \begin{align*}
    \P[\hat{i}_1=i]\leq \left(1+\frac{1}{2\UD_1^2}\right)\left(\frac{\delta}{\Ndelta}\right)^{\frac{1}{4}\left(\frac{\UD_i}{\max\{\theta_U,1-\theta_L\}}\right)^2}.
    \end{align*}
\end{lemma}
\begin{proof}
Define $\mathbb{F}_i(a)$ as $\mathbb{F}_i(a)=\P[Y_i\geq a]$. Then,
  \begin{align}
   \P[\hat{i}_1=i]
   =&\P\left[\hat{i}_1=i,Y_i\geq -\frac{\UD_i}{2}\sqrt{\ndelta}\right]
   +\P\left[\hat{i}_1=i,Y_i< -\frac{\UD_i}{2}\sqrt{\ndelta}\right]\nonumber\\
   \leq & \P\left[Y_i\geq -\frac{\UD_i}{2}\sqrt{\ndelta}\right]+\P\left[Y_1\leq Y_i,Y_i< -\frac{\UD_i}{2}\sqrt{\ndelta}\right].\label{ineq:14-1}
\end{align}
The second term is bounded as
\begin{align}
    & \P\left[Y_1\leq Y_i,Y_i< -\frac{\UD_i}{2}\sqrt{\ndelta}\right]\nonumber\\
= & \int_{-\frac{\UD_i}{2}\sqrt{\ndelta}}^{-\infty}\P[Y_1\leq Y_i\mid Y_i=a]\d \mathbb{F}_i(a)\nonumber\\
     \leq& \int_{-\frac{\UD_i}{2}\sqrt{\ndelta}}^{-\infty}\P[Y_1\leq a]\d \mathbb{F}_i(a)\nonumber\\
     \leq& \int_{-\frac{\UD_i}{2}\sqrt{\ndelta}}^{-\infty}\frac{e^{-2a^2}}{2\UD_1^2}\d \mathbb{F}_i(a) \ \ ( \text{by Lemma~\ref{lem:yia}})\nonumber\\
     =&\frac{1}{2\UD_1^2}\Biggl(\left[e^{-2a^2}\P[Y_i\geq a]\right]_{-\frac{\UD_i}{2}\sqrt{\ndelta}}^{-\infty}
     +\int_{-\frac{\UD_i}{2}\sqrt{\ndelta}}^{-\infty}4ae^{-2a^2}\P[Y_i\geq a]\d a\Biggr)\nonumber\\
    & \hspace*{6cm} (\text{using integration by parts})\nonumber\\
     \leq &\frac{1}{2\UD_1^2}\Biggl(-\e^{-\ndelta\frac{\Delta_i^2}{2}}\P\left[Y_i\geq -\frac{\UD_i}{2}\sqrt{\ndelta}\right]
     +\int_{-\frac{\UD_i}{2}\sqrt{\ndelta}}^{-\infty}4ae^{-2a^2}\d a \Biggr)\nonumber\\
     \leq &\frac{1}{2\UD_1^2}\int_{-\frac{\UD_i}{2}\sqrt{\ndelta}}^{-\infty}4ae^{-2a^2}\d a\nonumber\\     
     =&-\frac{1}{2\UD_1^2}\left[\e^{-2a^2}\right]_{-\frac{\UD_i}{2}\sqrt{\ndelta}}^{-\infty}
     =\frac{1}{2\UD_1^2}\e^{-\ndelta\frac{\UD_i^2}{2}}
     \leq\frac{1}{2\UD_1^2}\left(\frac{\delta}{\Ndelta}\right)^{\frac{1}{4}\left(\frac{\UD_i}{\max\{\theta_U,1-\theta_L\}}\right)^2}.\label{ineq:14-2}
  \end{align}
Thus, by Ineq. (\ref{ineq:14-1}), (\ref{ineq:14-2}) and Lemma~\ref{lem:yiaE*},
\[
\P[\hat{i}_1=i]\leq \left(1+\frac{1}{2\UD_1^2}\right)\left(\frac{\delta}{\Ndelta}\right)^{\frac{1}{4}\left(\frac{\UD_i}{\max\{\theta_U,1-\theta_L\}}\right)^2}
\]
holds.
\hspace*{\fill}$\Box$\end{proof}

\begin{lemma}\label{lem:CompEventProb}
  For the complementary events $\overline{\Ev_i^\mathrm{POS}}$ of event $\Ev_i^\mathrm{POS}$,
  inequality
\[
  \P\left[\overline{\Ev_i^\mathrm{POS}}\right]\leq \frac{\e^{2\UD_i^2}}{2\UD_i^2}\left(\frac{\delta}{\Ndelta}\right)^{\left(\frac{\UD_i}{\max\{\theta_U,1-\theta_L\}}\right)^2}
  \]
  holds when $i\leq m$.
\end{lemma}
\begin{proof}
  In the case with $\hat{\mu}_i(\tau_i)\geq\theta$, arm $i$ is judged as positive because
  $\LB{i}{\tau_i}\geq \theta_L$ holds whenever $\UB{i}{\tau_i}<\theta_U$ holds\footnote{An arm is judged as positive when both the stopping conditions are satisfied.}. This is because
  $\theta_U-\theta:\theta-\theta_L=\UB{i}{\tau_i}-\hat{\mu}_i(\tau_i):\hat{\mu}_i(\tau_i)-\LB{i}{\tau_i}=1:\alpha$ holds. Thus,
  \begin{align*}
    \P\left[\overline{\Ev_i^\mathrm{POS}}\right]\leq& \P\left[\bigcup_{n=\ndelta}^{\Td{\Delta}}\{\hat{\mu}_i(n)<\theta\}\right]\\
    =& \P\left[\bigcup_{n=\ndelta}^{\Td{\Delta}}\{\hat{\mu}_i(n)<\mu_i-\UD_i\}\right]\\
    \leq& \sum_{n=\ndelta}^{\Td{\Delta}}\P[\hat{\mu}_i(n)<\mu_i-\UD_i]\\
    \leq& \sum_{n=\ndelta}^{\infty}\e^{-2n\UD_i^2}
    = \frac{\e^{2\UD_i^2}\e^{-2\ndelta\UD_i^2}}{\e^{2\UD_i^2}-1}
    \leq \frac{\e^{2\UD_i^2}}{2\UD_i^2}\left(\frac{\delta}{\Ndelta}\right)^{\left(\frac{\UD_i}{\max\{\theta_U,1-\theta_L\}}\right)^2}
  \end{align*}  
holds.
\hspace*{\fill}$\Box$\end{proof}

\noindent
    {\it Proof of Theorem~\ref{th:apt-bc}}\ 
  \begin{align*}
    \E[T] =&\sum_{i=1}^m \E\left[T\mathbb{1}\left\{\hat{i}_1=i,\Ev_i^\mathrm{POS}\right\}\right]
     +\sum_{i=1}^m \E\left[T\mathbb{1}\left\{\hat{i}_1=i,\overline{\Ev_i^\mathrm{POS}}\right\}\right]
    + \sum_{i=m+1}^K \E[T\mathbb{1}\{\hat{i}_1=i\}]\\
  \leq &\sum_{i=1}^m\begin{aligned}[t]
    \Biggl( \frac{\P\left[\hat{i}_1=i,\Ev_i^\mathrm{POS}\right]}{2\Delta_i^2}\ln\frac{K\Ndelta}{\delta}
    &+O\left(\left(\ln\frac{K\Ndelta}{\delta}\right)^{\frac{2}{3}}\right) +\sum_{j\leq m,j\neq i}\frac{2}{\UD_j^4}\\
    &\ \ \ +\sum_{j=m+1}^K \left\{\frac{1}{\UD_i^2\UD_j^2}+\left(\frac{4}{\UD_j^2}+1\right)\right\}\Biggr)\\
    & \hspace*{3cm} ( \text{by Lemma~\ref{lem:i1iE*}})
  \end{aligned}\\
   &+K\Td{\Delta}\sum_{i=1}^m\mathbb{P}\left[\overline{\Ev_i^\mathrm{POS}}\right]
 +K\Td{\Delta}\sum_{i=m+1}^K\mathbb{P}[\hat{i}_1=i]\\
  \leq &\sum_{i=1}^m
    \Biggl(\frac{\P\left[\hat{i}_1=i,\Ev_i^\mathrm{POS}\right]}{2\Delta_i^2}\ln\frac{K\Ndelta}{\delta}\begin{aligned}[t]
    & +O\left(\left(\ln\frac{K\Ndelta}{\delta}\right)^{\frac{2}{3}}\right)+\frac{2(m-1)}{\UD_i^4}\\
    &+\left(\frac{1}{\UD_i^2}+4\right)\sum_{j=m+1}^K \frac{1}{\UD_j^2} +(K-m)\Biggr)
  \end{aligned} \\
  & \begin{aligned}[t]&+K\Td{\Delta}\Biggl(
    \frac{\e^{2\UD_i^2}}{2\UD_i^2}\sum_{i=1}^m\left(\frac{\delta}{\Ndelta}\right)^{\left(\frac{\UD_i}{\max\{\theta_U,1-\theta_L\}}\right)^2} \ \ \ \ ( \text{by Lemma~\ref{lem:CompEventProb}})\\
    &+\left(1+\frac{1}{2\UD_1^2}\right)\sum_{i=m+1}^K\left(\frac{\delta}{\Ndelta}\right)^{\frac{1}{4}\left(\frac{\UD_i}{\max\{\theta_U,1-\theta_L\}}\right)^2}\Biggr) \ \ \ \ ( \text{by Lemma~\ref{lem:m1im+1E*}})
    \end{aligned}\\
  = &\sum_{i=1}^m
    \Biggl( \frac{\P\left[\hat{i}_1=i,\Ev_i^\mathrm{POS}\right]}{2\Delta_i^2}\ln\frac{K\Ndelta}{\delta}
    +\frac{2(m-1)}{\UD_i^4}+\left(\frac{1}{\UD_i^2}+4\right)\sum_{j=m+1}^K \frac{1}{\UD_j^2}\Biggr)\\
  & +m(K-m)+O\left(m\left(\ln\frac{K\Ndelta}{\delta}\right)^{\frac{2}{3}}\right)\\
  & \begin{aligned}[t]&+K\Td{\Delta}\Biggl(
    \frac{\e^{2\UD_i^2}}{2\UD_i^2}\sum_{i=1}^m\left(\frac{\delta}{\Ndelta}\right)^{\left(\frac{\UD_i}{\max\{\theta_U,1-\theta_L\}}\right)^2}\\
    &+\left(1+\frac{1}{2\UD_1^2}\right)\sum_{i=m+1}^K\left(\frac{\delta}{\Ndelta}\right)^{\frac{1}{4}\left(\frac{\UD_i}{\max\{\theta_U,1-\theta_L\}}\right)^2}\Biggr)
    \end{aligned}\\
\end{align*}
  \hspace*{\fill}$\Box$

  \section{Proof of Theorem~\ref{th:ucb-bc}}\label{proof:ucb-bc}

  We consider event $\bigcup_{i:\mu_i=\mu_1}\Ev_i^\mathrm{POS}$, that is, the event that one of the best arm $i$ is judged as positive.
  In the case that event $\bigcup_{i:\mu_i=\mu_1}\Ev_i^\mathrm{POS}$ does not occur, 
stopping time $T$ is upper bounded by the worst case bound $K\Td{\Delta}$ (Theorem~\ref{thwelldef}) and the decreasing order of the occurrence probability of this case as $\delta\rightarrow +0$ can be proved to be small compared to the increasing order of $K\Td{\Delta}$ (Lemma~\ref{lem:CompEventProb}), so it can be ignored asymptotically as $\delta\rightarrow +0$.
When event $\bigcup_{i:\mu_i=\mu_1}\Ev_i^\mathrm{POS}$ occurs,
non-optimal arms $i$ with $\mu_i<\mu_1$ is drawn in the case of $\mu_i$'s overestimation ($\mathrm{UCB}(t,i)\geq \mu_1-\epsilon$)
or in the case of $\mu_1$'s underestimation ($\mathrm{UCB}(t,1)< \mu_1-\epsilon$).
So, $\E[T\mathbb{1}\{\bigcup_{i:\mu_i=\mu_1}\Ev_i^\mathrm{POS}]$ is upper bounded by
upper bounding $\E[\tau_i\mathbb{1}\{\bigcup_{i:\mu_i=\mu_1}\Ev_i^\mathrm{POS}\}]$ for optimal arms $i$ with $\mu_i=\mu_1$ by Lemma~\ref{lem:expectedtau}, the expected number of overestimations $\E\left[\sum_{t=1}^{K\Td{\Delta}}\mathbb{1}[\mathrm{UCB}(t,i)\geq \mu_1-\epsilon, i_t=i]\right]$
for non-optimal arms $i$ with $\mu_i<\mu_1$ by Lemma~\ref{lem:UcbOverEst},
and the expected number of underestimations $\E\left[\sum_{t=1}^{K\Td{\Delta}}\mathbb{1}[\mathrm{UCB}(t,1)< \mu_1-\epsilon]\right]$ for
the optimal arm $1$ by Lemma~\ref{lem:UcbUnderEst}.
  
  \begin{lemma}\label{lem:UcbOverEst}
    For an arbitrary $\epsilon>0$, $\baec{\mathrm{UCB}}{\underline{\mu}}{\overline{\mu}}$ satisfies
    \begin{align*}
\E\left[\sum_{t=1}^{K\Td{\Delta}}\mathbb{1}[\mathrm{UCB}(t,i)\geq \mu_1-\epsilon, i_t=i]\right]
\leq \frac{\ln K\Td{\Delta}}{2(\Delta_{1i}-2\epsilon)^2}+\frac{1}{2\epsilon^2}+1 
\end{align*}
for $i=2,...,K$ with $\mu_i<\mu_1$.
    \end{lemma}
  \begin{proof}
    Let $n'_i=\frac{\ln K\Td{\Delta}}{2(\Delta_{1i}-2\epsilon)^2}$. Then,
    \begin{align*}
    \sum_{t=1}^{K\Td{\Delta}}\mathbb{1}[\mathrm{UCB}(t,i)\geq \mu_1-\epsilon, i_t=i]
     = &
      \sum_{t=1}^{K\Td{\Delta}}\sum_{n=0}^{K\Td{\Delta}-1}\mathbb{1}\left[\hat{\mu}_i(n) +\sqrt{\frac{\ln t}{2n}}\geq \mu_1-\epsilon, n_i(t)=n, i_t=i\right]\\
= & \sum_{n=0}^{K\Td{\Delta}-1}\mathbb{1}\left[\bigcup_{t=1}^{K\Td{\Delta}}\left\{\hat{\mu}_i(n) +\sqrt{\frac{\ln t}{2n}}\geq \mu_1-\epsilon, n_i(t)=n, i_t=i\right\}\right]\\
      \leq & \sum_{n=0}^{K\Td{\Delta}-1}\mathbb{1}\left[\hat{\mu}_i(n) +\sqrt{\frac{\ln K\Td{\Delta}}{2n}}\geq \mu_1-\epsilon\right]\\
      \leq & \sum_{n=0}^{\lfloor n'_i\rfloor}1
      + \sum_{n=\lfloor n'_i\rfloor+1}^{\infty}\mathbb{1}\left[\hat{\mu}_i(n)+\sqrt{\frac{\ln K\Td{\Delta}}{2\cdot \frac{\ln K\Td{\Delta}}{2(\Delta_{1i}-2\epsilon)^2}}}\geq \mu_1-\epsilon\right]\\
      \leq & \frac{\ln K\Td{\Delta}}{2(\Delta_{1i}-2\epsilon)^2} +1+ \sum_{n=1}^{\infty}\mathbb{1}\left[\hat{\mu}_i(n)\geq \mu_i+\epsilon\right]
    \end{align*}
    Therefore,
    \begin{align*}
    \E\left[\sum_{t=1}^{K\Td{\Delta}}\mathbb{1}[\mathrm{UCB}(t,i)\geq \mu_1-\epsilon, i_t=i]\right]
    \leq & \frac{\ln K\Td{\Delta}}{2(\Delta_{1i}-2\epsilon)^2} +1+ \sum_{n=1}^{\infty}\P\left[\hat{\mu}_i(n)\geq \mu_i+\epsilon\right]\\
    = & \frac{\ln K\Td{\Delta}}{2(\Delta_{1i}-2\epsilon)^2} +1+ \sum_{n=1}^{\infty}\e^{-2n\epsilon^2}\\
    =& \frac{\ln K\Td{\Delta}}{2(\Delta_{1i}-2\epsilon)^2} +1+ \frac{1}{\e^{2\epsilon^2}-1}\\
    \leq & \frac{\ln K\Td{\Delta}}{2(\Delta_{1i}-2\epsilon)^2} + \frac{1}{2\epsilon^2} +1.
    \end{align*}   
    
  \hspace*{\fill}$\Box$\end{proof}

  \begin{lemma}\label{lem:UcbUnderEst}
    For $\baec{\mathrm{UCB}}{\underline{\mu}}{\overline{\mu}}$, the following inequality holds.
    \[
\E\left[\sum_{t=1}^{K\Td{\Delta}}\mathbb{1}[\mathrm{UCB}(t,1)< \mu_1-\epsilon]\right]\leq \frac{1}{\epsilon^2}+\frac{1}{4\epsilon^2}\ln\frac{1}{2\epsilon^2}
    \]
    for $0<\epsilon\leq 1$.
    \end{lemma}
  \begin{proof}
    \begin{align*}
     \sum_{t=1}^{K\Td{\Delta}}\mathbb{1}[\mathrm{UCB}(t,1)<\mu_1-\epsilon] = &\sum_{t=1}^{K\Td{\Delta}}\sum_{n=0}^{K\Td{\Delta}-1}\mathbb{1}\left[\hat{\mu}_1(n)+\sqrt{\frac{\ln t}{2n}}<\mu_1-\epsilon, n_1(t)=n\right]\\
      = &\sum_{n=0}^{K\Td{\Delta}-1}\sum_{t=1}^{K\Td{\Delta}}\mathbb{1}\Bigl[t<\e^{2n(\mu_1-\hat{\mu}_1(n)-\epsilon)^2},\hat{\mu}_1(n)<\mu_1-\epsilon, n_1(t)=n\Bigr]\\
      \leq & \sum_{n=1}^{K\Td{\Delta}-1}\e^{2n(\mu_1-\hat{\mu}_1(n)-\epsilon)^2}\mathbb{1}\left[\hat{\mu}_1(n)\leq\mu_1-\epsilon\right]\\
    \end{align*}
    Define $\mathbb{F}_n(x)$ as $\mathbb{F}_n(x)=\P\{\hat{\mu}_1(n)\leq x\}$. Note that $\mathbb{F}_n(x)\leq \e^{-2n(\mu_1-x)^2}$ for $x<\mu_1$ by Hoeffding's Inequality. Then,
\begin{align*}
  &  \E\left[\sum_{t=1}^{K\Td{\Delta}}\mathbb{1}[\mathrm{UCB}(t,1)< \mu_1-\epsilon]\right]\\
  \leq & \sum_{n=1}^{K\Td{\Delta}-1}\E\left[\e^{2n(\mu_1-\hat{\mu}_1(n)-\epsilon)^2}\mathbb{1}\left[\hat{\mu}_1(n)\leq \mu_1-\epsilon\right]\right]\\
  = & \sum_{n=1}^{K\Td{\Delta}-1}\int_{-\infty}^{\mu_1-\epsilon}\e^{2n(\mu_1-x-\epsilon)^2}\d\mathbb{F}_n(x)\\
  = & \sum_{n=1}^{K\Td{\Delta}-1} \biggl(\left[\e^{2n(\mu_1-x-\epsilon)^2}\mathbb{F}_n(x)\right]_{-\infty}^{\mu_1-\epsilon}
    + \int_{-\infty}^{\mu_1-\epsilon}\!\!\!\!\!\!\!\!\!4n(\mu_1-x-\epsilon)\e^{2n(\mu_1-x-\epsilon)^2}\mathbb{F}_n(x)\d x\biggr)\\
    \leq & \sum_{n=1}^{K\Td{\Delta}-1} \biggl(\mathbb{F}_n(\mu_1-\epsilon)+ \int_{-\infty}^{\mu_1-\epsilon}\!\!\!\!\!\!4n(\mu_1-x-\epsilon)\e^{2n(\mu_1-x-\epsilon)^2}\e^{-2n(\mu_1-x)^2}\d x\biggr)\\
    \leq & \sum_{n=1}^{K\Td{\Delta}-1} \biggl(\e^{-2n\epsilon^2}+ \int_{-\infty}^{\mu_1-\epsilon}\!\!\!\!\!\!4n(\mu_1-x-\epsilon)\e^{-2n\epsilon(2\mu_1-2x-\epsilon)}\d x\biggr)\\
    = & \sum_{n=1}^{K\Td{\Delta}-1} \biggl(\e^{-2n\epsilon^2} + \frac{1}{4n\epsilon^2}\left[\{4n\epsilon(\mu_1-x-\epsilon)+1\}\e^{-2n\epsilon(2\mu_1-2x-\epsilon)}\right]_{-\infty}^{\mu_1-\epsilon}\biggr)\\
    = & \sum_{n=1}^{K\Td{\Delta}-1} \biggl(\e^{-2n\epsilon^2} + \frac{1}{4n\epsilon^2}\e^{-2n\epsilon^2}\biggr)\\
    \leq & \frac{1}{e^{2\epsilon^2}-1}+\frac{-\ln(1-e^{-2\epsilon^2})}{4\epsilon^2}\ \ \ \left(\text{because} \sum_{n=1}^{\infty}\frac{\left(\e^{-2\epsilon^2}\right)^n}{n}=-\ln\left(1-\e^{-2\epsilon^2}\right)\right)\\
    \leq &\frac{1}{2\epsilon^2}+\frac{2\epsilon^2+\ln\frac{1}{e^{2\epsilon^2}-1}}{4\epsilon^2}\\
    \leq &\frac{1}{2\epsilon^2}+\frac{1}{2}+\frac{1}{4\epsilon^2}\ln\frac{1}{2\epsilon^2}
    \leq \frac{1}{\epsilon^2}+\frac{1}{4\epsilon^2}\ln\frac{1}{2\epsilon^2}
\end{align*}
    
    \hspace*{\fill}$\Box$\end{proof}

  {\it Proof of Theorem~\ref{th:ucb-bc}}\\
Let $\epsilon$ be $0<\epsilon\leq \min_{i:\Delta_{1i}>0} \Delta_{1i}/4$.
  \begin{align*}
   \E[T] =& \E\left[T\mathbb{1}\left\{\bigcup_{i:\mu_i=\mu_1}\Ev_i^\mathrm{POS}\right\}\right]+\E\left[T\mathbb{1}\left\{\bigcap_{i:\mu_i=\mu_1}\overline{\Ev_i^\mathrm{POS}}\right\}\right]\\
    \leq&\E\left[\sum_{i:\mu_i=\mu_1}\tau_i\mathbb{1}\left\{\bigcup_{i:\mu_i=\mu_1}\Ev_i^\mathrm{POS}\right\}\right]+\E\left[\sum_{t=1}^{K\Td{\Delta}}\mathbb{1}\left\{\mu_{i_t}<\mu_1,\bigcup_{i:\mu_i=\mu_1}\Ev_i^\mathrm{POS}\right\}\right]\\&+\E\left[T\mathbb{1}\left[\bigcap_{i:\mu_i=\mu_1}\overline{\Ev_i^\mathrm{POS}}\right]\right]\\
    \leq &\sum_{i:\mu_i=\mu_1}\E\left[\tau_i\mathbb{1}\left\{\bigcup_{i:\mu_i=\mu_1}\Ev_i^\mathrm{POS}\right\}\right]\\
    &+\E\Biggl[\sum_{t=1}^{K\Td{\Delta}}\mathbb{1}\Bigl[\{\mathrm{UCB}(t,i_t)\geq \mu_1-\epsilon,\mu_{i_t}<\mu_1\}\cup\{\mathrm{UCB}(t,1)<\mu_1-\epsilon\}\Bigr]\Biggr]+\E\left[T\mathbb{1}\left\{\overline{\Ev_1^\mathrm{POS}}\right\}\right]\\
    \leq &\sum_{i:\mu_i=\mu_1}\E\left[\tau_i\mathbb{1}\left\{\bigcup_{i:\mu_i=\mu_1}\Ev_i^\mathrm{POS}\right\}\right]
    +\sum_{i:\mu_i<\mu_1}\E\left[\sum_{t=1}^{K\Td{\Delta}}\mathbb{1}\left[\mathrm{UCB}(t,i)\geq \mu_1-\epsilon,i_t=i\}\right]\right]\\
    &+\E\left[\sum_{t=1}^{K\Td{\Delta}}\mathbb{1}\left[\mathrm{UCB}(t,1)<\mu_1-\epsilon\right]\right]+\E\left[T\mathbb{1}\left\{\overline{\Ev_1^\mathrm{POS}}\right\}\right]\\
  \leq &\sum_{i:\mu_i=\mu_1}
    \left(\frac{\P\left[\bigcup_{i:\mu_i=\mu_1}\Ev_i^\mathrm{POS}\right]}{2\Delta_i^2}\ln\frac{K\Ndelta}{\delta}+O\left(\left(\ln\frac{K\Ndelta}{\delta}\right)^{\frac{2}{3}}\right) \right)\ \ \ (\text{by Lemma~\ref{lem:expectedtau}})\\
    & + \sum_{i:\mu_i<\mu_1}\left(\frac{\ln K\Td{\Delta}}{2(\Delta_{1i}-2\epsilon)^2}+\frac{1}{2\epsilon^2}+1 \right)+\frac{1}{\epsilon^2}+\frac{1}{4\epsilon^2}\ln\frac{1}{2\epsilon^2}+K\Td{\Delta}\P\left[\overline{\Ev_1^\mathrm{POS}}\right].\\
    & \hspace*{5cm} (\text{by Lemma~\ref{lem:UcbOverEst} and \ref{lem:UcbUnderEst}, and Theorem~\ref{thwelldef}})
  \end{align*}
Since $\frac{1}{\Delta_{1i}^2}+\frac{12\epsilon}{\Delta_{1i}^3}- \frac{1}{(\Delta_{1i}-2\epsilon)^2}=\frac{4\epsilon(\Delta_{1i}-4\epsilon)(2\Delta_{1i}-3\epsilon)}{\Delta_{1i}^3(\Delta_{1i}-2\epsilon)^2}\geq 0$ holds for $0<\epsilon\leq \frac{\Delta_{1i}}{4}$,
$\frac{1}{(\Delta_{1i}-2\epsilon)^2}\leq \frac{1}{\Delta_{1i}^2}+\frac{12\epsilon}{\Delta_{1i}^3}$ holds.
Thus, by setting $\epsilon$ to $O((\ln K\Td{\Delta})^{-1/3})$, we have
  \begin{align*}
    \E[T]\leq &
    \sum_{i:\mu_i=\mu_1}
    \Biggl(\frac{1}{2\Delta_i^2}\ln\frac{K\Ndelta}{\delta}+O\left(\left(\ln\frac{K\Ndelta}{\delta}\right)^{\frac{2}{3}}\right) \Biggr)+ \sum_{i:\mu_i<\mu_1}\left(\frac{\ln K\Td{\Delta}}{2\Delta_{1i}^2}+O((\ln K\Td{\Delta})^{\frac{2}{3}}) \right)\\
    &+O((\ln K\Td{\Delta})^{\frac{2}{3}}\ln\ln K\Td{\Delta})+\frac{\e^{2\UD_1^2}K\Td{\Delta}}{2\UD_1^2}\left(\frac{\delta}{\Ndelta}\right)^{\left(\frac{\UD_1}{\max\{\theta_U,1-\theta_L\}}\right)^2}.\ \ \ (\text{by Lemma~\ref{lem:CompEventProb}})
    \end{align*}
  \hspace*{\fill}$\Box$

\end{document}